\documentclass[10pt]{article} % For LaTeX2e
\usepackage[accepted]{tmlr}

\usepackage{hyperref}
\usepackage{url}

\usepackage{amsmath,amsfonts}
\usepackage{algorithmic}
\usepackage{algorithm}
\usepackage{array}
\usepackage[caption=false,font=normalsize,labelfont=sf,textfont=sf]{subfig}
\usepackage{textcomp}
\usepackage{stfloats}
\usepackage{url}
\usepackage{verbatim}
\usepackage{graphicx}
\usepackage{amsmath,amsfonts,amssymb,amsthm}
\usepackage{bbm}
\usepackage{algorithmic}
\usepackage{algorithm}
\usepackage{array}
\usepackage[caption=false,font=normalsize,labelfont=sf,textfont=sf]{subfig}
\usepackage{textcomp}
\usepackage{stfloats}
\usepackage{url}
\usepackage{verbatim}
\usepackage{graphicx}
\usepackage{float}
\usepackage{balance}
\usepackage{xcolor}
\usepackage{placeins}
\usepackage{varwidth}
\usepackage{algorithmic}
\usepackage{enumitem}
\usepackage{adjustbox}

\newtheorem{theorem}{Theorem}[section]

\newtheorem{lemma}[theorem]{Lemma}

\newtheorem{definition}[theorem]{Definition}
\newtheorem{assumption}[theorem]{Assumption}
\newtheorem{remark}[theorem]{Remark}

\def\de{=}

\DeclareMathOperator*{\esssup}{ess\,sup}

\newif\iftwocol
\twocolfalse

\iftwocol
\else
    \onecolumn
\fi
\newcommand{\maybeNewline}{\iftwocol\\\fi}

\newcommand\eqnum{\addtocounter{equation}{1}\tag{\theequation}}

\title{Diffusion-Based Hypothesis Testing and Change-Point Detection}

\author{\name Sean Moushegian \email sean.moushegian@duke.edu\\ \addr Department of Electrical and Computer Engineering \\ Duke University
      \AND
      \name Taposh Banerjee \email taposh.banerjee@pitt.edu\\ 
      \addr Department of Industrial Engineering \\ University of Pittsburgh
      \AND 
      \name Vahid Tarokh \email vahid.tarokh@duke.edu\\ \addr Department of Electrical and Computer Engineering \\ Duke University
}

\def\month{10}  % Insert correct month for camera-ready version
\def\year{2026} % Insert correct year for camera-ready version

\def\openreview{\url{https://openreview.net/forum?id=gaf6vUQSSr}} 

\begin{document}

\maketitle

\begin{abstract}
Score-based methods have recently seen increasing popularity in modeling and generation.  Methods have been constructed to perform hypothesis testing and change-point detection with score functions, but these methods are in general not as powerful as their likelihood-based peers.  
Recent works consider generalizing the score-based Fisher divergence into a diffusion-divergence by transforming score functions via multiplication with a matrix-valued function or a weight matrix.  
In this paper, we extend the score-based hypothesis test and change-point detection stopping rule into their diffusion-based analogs. Additionally, we theoretically quantify the performance of these diffusion-based algorithms and study scenarios where optimal performance is achievable.   We propose a method of numerically optimizing the weight matrix and present numerical simulations to illustrate the advantages of diffusion-based algorithms.
\end{abstract}

\section{Introduction}
In many engineering problems, one seeks to infer whether some data $X$ was generated from a null distribution $P_\infty$ or an alternate distribution $P_1$.  
This question underpins the detection problems of hypothesis testing and anomaly detection
(\cite{lopez-paz2017revisiting, liu2020learning, angiulli2026reconstruction, optimality_binning_distributed_hypo_test, np_test_zero_rate_multiterminal_hypo_test})
and change-point detection 
(\cite{kei2025change, yamada2025change, unnikrishnan2011minimax, qcd_time_varying, qcd_change_time, byzantine_qcd, ar_model_qcd}).
These problems have been extensively studied under the assumption that the densities of $P_\infty$ and $P_1$ are known (or can be learned from data).  In this setting, the log-likelihood ratio (LLR) test \eqref{eq:test_llr} and cumulative sum (CUSUM) stopping rule \eqref{eq:cusum} are commonly used and optimal under the metrics of \eqref{eq:objective_hypothesis_testing} and \eqref{eq:cpd_objective}.
Further background on hypothesis testing and change-point detection is provided in Section \ref{sec:hypo_cpd_background}.

Score-based models (\cite{song2019score, song_sliced_score_matching, melidonis2025scorebased, berman2025reviving, kollovieh2024assessing}) directly estimate the score $s(X) \de \nabla_X \log p(X) : \mathbb{R}^d \mapsto \mathbb{R}^d$ of a distribution $P$ by training over a dataset of independent and identically distributed (i.i.d.) samples of $P$.  Furthermore, Markov Chain Monte Carlo methods  (\cite{roberts_rosenthal, gbrbm_without_tears, aloui2025scorebasedmetropolishastingsalgorithms}) have been developed to sample from $P$ using $\nabla_X \log p(X)$.  
Recent works have proposed \textit{score-based} (Definition~\ref{def:fisher_divergence}) methods of hypothesis testing (\cite{wu2022score}) and change-point detection (\cite{wu_scusum_aistats}) under the assumption that only the score functions of $P_\infty, P_1$ are known.  Further background on score-based models is provided in Section \ref{sec:score_background}.

Under a popular objective for simple hypothesis testing \eqref{eq:objective_hypothesis_testing} and change-point detection \eqref{eq:cpd_objective}, the LLR test and CUSUM stopping rule are provably optimal \textit{for any bound on false alarm} (and in the case of hypothesis testing, for any batch size).  Their score-based counterparts, however, are not in general optimal, and for specific choices of $P_\infty, P_1$, can be shown to strictly underperform the LLR-based algorithms.
Even so, the score-based methods are merely one set of algorithms that can test hypotheses and detect change-points using score functions.  We investigate whether some other choice of algorithms which use the score functions can better compete with the LLR-based methods.

To do so, we first present the \textit{diffusion divergence}, studied by \cite{barp2019minimum} and applied to detection tasks by \cite{altamirano2023robust}.  The diffusion divergence generalizes the Fisher divergence by transforming scores $\nabla_X \log p_\infty(X), \nabla_X \log p_1(X)$ by multiplication with a matrix-valued function $m(X) : \mathbb{R}^d \mapsto \mathbb{R}^{d \times w}$ which we refer to as the \textit{diffusion matrix function}.  The score-based algorithms are one special case of the diffusion algorithms where the identity matrix $I \in \mathbb{R}^{d \times d}$ is chosen to be the diffusion matrix function $m(X)$.  

The performance of the diffusion algorithms will depend upon their particular choice of $m(X)$.
Using the metrics of \eqref{eq:objective_hypothesis_testing}, \eqref{eq:cpd_objective}, a good choice of $m(X)$ will produce diffusion algorithms which perform at least as well as the score-based algorithms and no better than the  (provably optimal) LLR-based algorithms.  
Conversely, diffusion algorithms can perform worse than even score-based algorithms for poorly chosen $m(X)$.
Thus, the selection of $m(X)$ so that it is optimal in a well-defined sense is an important goal of our work.

\paragraph{Related Work}

\cite{barp2019minimum} constructed the diffusion divergence and applied it to estimation. Specifically, \cite{barp2019minimum} formulated the diffusion divergence (defined in Definition~\ref{def:diffusion_divergence})
as a Stein discrepancy, proposed a finite-sample version of the divergence, which is used as a generalized score-matching objective, and employed this objective to estimate the parameters of a model class conditioned upon data.

\cite{altamirano2023robust} applied this diffusion divergence to propose a method of change detection that emphasizes robustness and scalability. \cite{altamirano2023robust} considered data streams with potentially many change-points, in which the (unknown) parameter of the data generating distribution changes at each change-point. Approaching change-point detection through a Bayesian formulation, \citet{altamirano2023robust} estimated for each time $t$
the posterior distribution $p(r_t | x_1, \cdots, x_t)$
of the time $r_t$
elapsed since the most recent change-point.  From this posterior distribution, change-points were reported following MAP segmentation obtained via a Viterbi-style recursion.  \citet{altamirano2023robust} employed the diffusion divergence using a diagonal matrix-valued function $m(X)$.

This paper considers a distinct problem from the formulation in \cite{altamirano2023robust}, and differs methodologically in several ways.  We assume exactly one change point, assume knowledge of the pre- and post-change score functions, and define a diffusion-based instantaneous detection score for CUSUM-style detection.  As such, our method is a stopping rule which requires no backwards pass or MAP segmentation to detect the change-point.
Furthermore, we implement our matrix-valued function $m(X)$ as a neural network and iteratively train this neural network to optimize for performance.  Finally, we use the diffusion divergence to propose a hypothesis test.

\paragraph{Summary of Contributions}

We now summarize the main contributions of this paper:
\begin{enumerate}
    \item We propose a hypothesis test and a change-point detection stopping rule that are based on the diffusion divergence, generalizing the score-based hypothesis test of \cite{wu2022score} and change-point detection stopping rule of \cite{wu_scusum_aistats} by way of a \textit{diffusion matrix function} $m(X)$.  These algorithms are proposed in Section \ref{sec:diffusion_detection}.
    \item We bound (or calculate) the error exponent, expected detection delay, and mean time to false alarm of the diffusion-based algorithms asymptotically and as a function of $m(X)$.  These properties are presented in Section \ref{sec:diffusion_detection}.
    \item For each detection task, we present an optimization objective over $m(X)$. 
    We present this objective for change-point detection in Section \ref{sec:optimization_objective_cpd} and for  hypothesis testing in Section \ref{sec:optimization_objective_hypo}.
    \item We show that for special choices of $P_\infty, P_1$,  
    there exists an $m(X)$ such that the diffusion-based test statistic $Z_m(X)$ is equal to the KL-based statistic $Z_\texttt{\textup{KL}}(X)$ for all $X \in \mathbb{R}^d$. Thus, for these choices of $P_\infty, P_1$, the diffusion-based tests are optimal and equal to the LLR-based counterparts.  
 We show that for other choices of $P_\infty, P_1$, there is no $m(X)$ for which $Z_m(X) = Z_{\texttt{\textup{KL}}}(X)$ with probability one under the measures of $P_\infty$ or $P_1$.  We argue that the best diffusion algorithm for either detection task will perform no worse than the score-based algorithm and no better than the LLR algorithm.
    We present this analysis in Section \ref{sec:optimization_analytical}.
\end{enumerate}

\medskip

\paragraph{Organization of Paper}
The remainder of this paper is organized as follows: in Section \ref{sec:score_background}, we provide background on score-based models.  In Section \ref{sec:hypo_cpd_background}, we provide background on the detection problems of hypothesis testing and change-point detection, and discuss LLR-based and score-based approaches to these problems.  In Section~\ref{sec:diffusion_detection}, we propose a diffusion-based hypothesis test and change-detection algorithm and analyze the properties of these algorithms.  In Section~\ref{sec:optimization_numerical}, we propose a differentiable loss function for the numerical optimization of $m(X)$.  In Section~\ref{sec:simulations}, we present numerical simulations to illustrate the performance and implementability of our proposed algorithms, and we provide the implementation details of these simulations in Appendix~\ref{appendix:details_sims}.  Finally, in Appendix \ref{appendix:proofs}, we provide proofs for the Theorems and Lemmas presented throughout the paper.

Throughout this paper, for any distribution $P$, we shall refer to the density as $p(X)$, the corresponding probability measure as $\mathbb{P}_P$, and the expectation as $\mathbb{E}_P$.

\section{Score-Based Models}
\label{sec:score_background}

In this section, we provide background on score-based models, which utilize the score $\nabla_X \log p(X)$ of a distribution $P$.  
We begin by defining the Fisher divergence between two distributions:
\begin{definition}[Fisher Divergence]
\label{def:fisher_divergence}
The Fisher divergence between distributions $P$ and $Q$ is defined to be:
\begin{equation*}\label{eq:FisherDivDef}
\mathbb{D}_{\texttt{\textup{F}}}(P \| Q) \de\mathbb{E}_{X\sim P} \left[ \tfrac{1}{2} \left\| \nabla_{{X}} \log p(X)- \nabla_{{X}} \log q(X)\right \|_2^2 \right].
\end{equation*}
\end{definition} 
We next define the Hyv\"arinen score:
\begin{definition}[Hyv\"arinen Score]
\label{def:hyvarinen_score}
We define the Hyv\"arinen score of $X$ with respect to distribution $Q$ as:
\begin{equation}
    \label{eq:hyv_score}
    \mathcal{S}_{\texttt{\textup{H}}}(X, Q) \de \tfrac{1}{2} \left \| \nabla \log q(X) \right \|_2^2 + \Delta_{X} \log q(X),
\end{equation}  
where $\Delta_X$ is the Laplacian operator: $\Delta_X f(X) \de \sum_{i=1}^d {\partial^2 f(X)} / {\partial X_i^2}$.
\end{definition}

\begin{lemma}[Hyv\"arinen's Theorem]
\label{lemma:hyvarinen_fisher}
Let $s(X) : \mathbb{R}^d \mapsto \mathbb{R}^d$.  Under the assumption that $\nabla_X \log p(X)$ and $s(X)$ are differentiable with a continuous derivative, that $\nabla_X \log p(X) \rightarrow 0$ and $s(X) \rightarrow 0$ as $\|X\| \rightarrow \infty$, and that both $\mathbb{E}_P[\|\nabla_X \log p(X)\|^2]$ and $\mathbb{E}_P[\| s(X) \|^2]$ are finite, we have that:
    \begin{equation}
    \label{eq:hyv_res}
    \mathbb{E}_{P}\left[ \tfrac{1}{2} \| \nabla_X \log p(X) - s(X) \|^2 \right] = \mathbb{E}_{P}\left[ \tfrac{1}{2} \|\nabla_X \log p(X)\|^2 + \tfrac{1}{2}\| s(X)  \|^2 + tr(\nabla_X s(X)) \right].
    \end{equation}  
    If for some distribution $Q$ with density $q(X)$ we set $s(X) = \nabla_X \log q(X)$, then we have that
    \begin{equation}
    \label{eq:hyv_thm_2}
    \mathbb{D}_\texttt{\textup{F}}(P \| Q) = \mathbb{E}_P\left[ \tfrac{1}{2}\| \nabla_X \log p(X) \|^2\right] + \mathbb{E}_P \left[\mathcal{S}_\texttt{\textup{H}}(X, Q)\right] .
    \end{equation}
\end{lemma}
\begin{proof}
    The proof is given in Appendix~\ref{appendix:proofs}.
\end{proof}

If we do not know the density or score function of some distribution $P$, but do have a dataset $(X_i)_{i=1}^n$ of i.i.d. samples of $P$, we can use score matching  (\cite{song2019score, song_sliced_score_matching, vincent2011connection}) to estimate $\nabla_X \log p(X)$ by $s_\theta(X) : \mathbb{R}^d \mapsto \mathbb{R}^d$, where $\theta$ parameterizes $s$.
Score matching optimizes
\begin{equation}
   \min_\theta \sum_{i=1}^n J_\theta(X_i) \quad \text{ where }  \quad J_\theta(X) = \left(\tfrac{1}{2}\| s_\theta(X) \|^2 + tr(\nabla_X s_\theta(X))\right).
\end{equation}
We note that in the limit of large $n$, $\frac{1}{n} \sum_{i=1}^n J_\theta(X_i)$  converges to $\mathbb{E}_P [J_\theta(X)]$,
and that if there exists a distribution $Q_\theta$ for which $s_\theta(X) = \nabla_X \log q_\theta(X)$, then $\mathbb{E}_P[J_\theta(X)] = \mathbb{E}_P[\mathcal{S}_\texttt{\textup{H}}(X, Q_\theta)]$.  In this case, score matching produces the $\theta$ which minimizes $\mathbb{D}_\texttt{\textup{F}}(P \| Q_\theta)$ by \eqref{eq:hyv_thm_2}.  We note that score matching is analogous to maximum likelihood estimation when $\mathbb{D}_\texttt{\textup{F}}$ and $\mathcal{S}_\texttt{\textup{H}}$ are replaced by $\mathbb{D}_\texttt{\textup{KL}}$ and negative-log-likelihood, respectively.

\section{Score-Based Hypothesis Testing and Change-Point Detection}
\label{sec:hypo_cpd_background}

We now present background on hypothesis testing and change-point detection, which we shall refer to as our \textit{detection tasks}. Both detection tasks require us to discern between two competing hypotheses: hypothesis $\mathcal{H}_\infty$, under which the data was generated by $P_\infty$, and hypothesis $\mathcal{H}_1$, under which the data was generated by $P_1$.  For both tasks, we discuss LLR-based and score-based solutions.

In this section and throughout the paper, we shall make the following assumptions regarding $P_\infty$ and $P_1$:
\begin{assumption}
\label{assumption:p_inf_p_1}
We assume that:
        (i) $P_\infty$ and $P_1$ are both supported on $\mathbb{R}^d$,
        (ii) $P_\infty \neq P_1$,
        (iii) $P_\infty, P_1$ are absolutely continuous with respect to the Lebesgue measure, and
        (iv) The densities $p_\infty, p_1$ are twice differentiable with continuous derivatives.
\end{assumption}

\subsection{Simple Hypothesis Testing}
\label{sec:background_hypo_test}
Suppose $n$ samples $( X_i)_{i=1}^n $ are drawn independently from some distribution $P_*$: under the null hypothesis $\mathcal{H}_\infty$, $P_* = P_\infty$ and under the alternate hypothesis $\mathcal{H}_1$, $P_* = P_1$.
The task of simple hypothesis testing (hereafter referred to as hypothesis testing) is to decide whether or not to reject $\mathcal{H}_\infty$.

We define a \textit{test} $T$ to be a function that maps data $( X_i)_{i=1}^n$ to a decision: $T(( X_i)_{i=1}^n) \in \{0, 1\}$.  We reject $\mathcal{H}_\infty$ if $T(( X_i)_{i=1}^n) = 1$, and fail to reject $\mathcal{H}_\infty$ otherwise.  A type I (type II) error occurs when the test incorrectly rejects (incorrectly fails to reject) $\mathcal{H}_\infty$.  We denote the probabilities of type I and type II error as
\begin{equation}
\label{eq:err_prob}
 \alpha_n(T) \de \mathbb{P}_{\infty}[T(( X_i)_{i=1}^n) = 1],\quad\quad \beta_n(T) \de \mathbb{P}_{1}[T(( X_i)_{i=1}^n) = 0]
\end{equation}
respectively.  We refer to $n$ as the \textit{batch size} of test $T$.

The work of this paper is motivated by the Neyman-Pearson regime, in which one cannot tolerate a probability of type I error in excess of some threshold $\overline{\alpha}$.  Thus, we wish to find a test $T$ that is a solution of (or at least \textit{approximates} a solution of):
\begin{equation}
\label{eq:objective_hypothesis_testing}
    \min_{T} \beta_n(T) \; \text{ subject to } \; \alpha_n(T) \leq \overline{\alpha}.
\end{equation}

The \textit{type II error exponent} describes the relationship between batch size and type II error probability for large $n$.
\begin{definition}[Type II Error Exponent]
\label{def:error_exponent}
For a test $T$, the type II error exponent $E$ is given by
\begin{equation}
    E =  -\limsup_{n \rightarrow \infty}  \frac{\log(\beta_n(T))}{n}.
\end{equation}
We refer to the type II error exponent of any test $T$ as $\mathcal{B}(T)$.
\end{definition}

We next proceed to consider two solutions to the problem of hypothesis testing.

\subsubsection{Log-Likelihood Ratio Test}
When one has complete knowledge of the densities of $P_\infty, P_1$, one can use the log-likelihood ratio test:
\begin{equation}
\label{eq:test_llr}
T_{\texttt{\textup{KL}}}^c(( X_i)_{i=1}^n) \de \begin{cases} 
      0 & \text{ if } \; \sum_{i=1}^n Z_{\texttt{\textup{KL}}}(X_i) < c \\
      1 & \text{ else,}
\end{cases}
\quad
\text{ where }
\quad
Z_{\texttt{\textup{KL}}}(X) \de \log p_1(X) - \log p_\infty(X),
\end{equation}
and where $c \in \mathbb{R}$ is some pre-determined threshold. Under the assumption that $\sum_{i=1}^n Z_{\texttt{\textup{KL}}}(X_i) \neq c $ almost surely, the Neyman-Pearson Lemma proves that the log-likelihood ratio test of \eqref{eq:test_llr} is optimal under the objective of \eqref{eq:objective_hypothesis_testing} for all choices of $\overline{\alpha} \in (0, 1)$ and for all batch sizes $n \in \mathbb{N}$ (\cite{neyman_pearson, cover_thomas, vantrees_detection}).

\subsubsection{Hyv\"arinen Score Test}

In \cite{wu2022score}, a score-based hypothesis test which utilizes the Hyv\"arinen score was proposed.  It is defined as
\begin{equation}
\label{eq:test_hscore}
T_{\texttt{\textup{F}}}^c(( X_i)_{i=1}^n) \de \begin{cases} 
      0 & \text{ if } \;\sum_{i=1}^n Z_{\texttt{\textup{F}}}(X_i) < c \\
      1 & \text{ else,}
\end{cases}
\quad
\text{ where }
\quad 
Z_{\texttt{\textup{F}}}(X) \de \mathcal{S}_{\texttt{\textup{H}}}(X, P_\infty) - \mathcal{S}_{\texttt{\textup{H}}}(X, P_1),
\end{equation}
and where $c \in \mathbb{R}$ is some pre-determined threshold.

A composite hypothesis test was proposed in \cite{wu2022score} for the case where only the null distribution is precisely known. Furthermore, \cite{wu2022score} derived limiting distributions on $\frac{1}{n} \sum_{i=1}^n Z_{\texttt{\textup{F}}}(X_i)$ for $X_i \overset{iid}{\sim} P_\infty$ as $n \rightarrow \infty$, and provided conditions under which this limiting distribution (after scaling by $\sqrt{n}$ and shifting) is the standard normal distribution.

\subsection{Online Change-Point Detection}
\label{sec:background_cpd}

\nocite{tartakovsky2005general}
\nocite{tartakovsky2014sequential}

We next consider a problem in which we must detect between $\mathcal{H}_\infty, \mathcal{H}_1$ under temporal constraints.
 Suppose we are given probability distributions $P_\infty, P_1$ and a natural number $\nu \in \mathbb{N}$.  A data-stream $(X_i)^\infty_{i=1}$ indexed by time $i \in \mathbb{N}$ can be generated by the following process:
\begin{equation}
\label{eq:data_stream_generation}
X_i \overset{iid}{\sim} P_\infty \text{ for } i < \nu,\quad\quad X_i \overset{iid}{\sim} P_1 \text{ for } i \geq \nu.
\end{equation}
We observe the data stream in real-time: at any time $i$, we can see past and present data, $(X_j)_{j \leq i}$, but cannot see future data, $(X_j)_{j > i}$.  Our objective is to raise an alarm as soon as possible after time $\nu$, but we wish to avoid prematurely raising the alarm before time $\nu$.

While hypothesis testing used a \textit{function} $T$, here we employ a \textit{stopping rule} $\tau$.
A \textit{false alarm} is the event where $\tau$ detects the change-point prematurely: $\tau < \nu$. In hypothesis testing, we considered the probability of type I error; in this setting, for any non-trivial stopping rule, the probability of a false alarm depends upon $\nu$ and can be made arbitrarily close to one (zero) for high (low) values of $\nu$, so we instead consider the \textit{mean time to false alarm}, a term which we shall use interchangeably with \textit{average run length}.  This quantity is given by:
\begin{equation}
\label{eq:def_arl}
    \text{ARL}(\tau) \de \mathbb{E}_\infty[\tau].
\end{equation}

Most non-trivial stopping rules will accumulate evidence of a change for some duration of time after $\nu$.  In this setting, we consider the mean time to alarm following $\nu$ rather than the probability of instantaneous detection. 
We define the worst-case average detection delay (\cite{lorden1971procedures}), which we shall also refer to as expected detection delay, by
\begin{align}
\label{eq:wadd}
\mathcal{L}_{\texttt{\textup{WADD}}}(\tau) &\de \sup_{\nu \geq 1} \esssup \mathbb{E}_\nu [(\tau-\nu + 1)^+ | \mathcal{F}_{\nu-1}],
\end{align}
where $\mathbb{E}_\nu$ refers to the expectation over data streams following \eqref{eq:data_stream_generation} for a particular $\nu$.  We seek to minimize the expected detection delay subject to a constraint on acceptable average run length $\overline{\gamma}$, as in \cite{wu_scusum_aistats}.  Thus, we wish to find a stopping rule that is a solution of (or at least \textit{approximates} a solution of):
\begin{equation}
\label{eq:cpd_objective}
    \min_\tau \; \mathcal{L}_{\texttt{\textup{WADD}}}(\tau) \; \text{ subject to } \; \text{ARL}(\tau) \geq \overline{\gamma}.
\end{equation}
Next, we consider one likelihood-based and one score-based solution to change-point detection.

\subsubsection{Cumulative Sum (CUSUM) Algorithm}
The Cumulative Sum (CUSUM) algorithm (\cite{page1955test, lorden1971procedures, moustakides1986optimal}) is a popular solution to the problem of change-point detection which is optimal under the objective of \eqref{eq:cpd_objective}.
The algorithm defines an \textit{instantaneous detection score}, $Z_{\texttt{\textup{KL}}}(X)$, as presented in \eqref{eq:test_llr}.  It further defines a \textit{cumulative detection score} $Y_\texttt{KL}(t)$ and \textit{stopping rule} $\tau_\texttt{KL}^c$ by
\begin{equation}
\label{eq:cusum}
Y_{\texttt{\textup{KL}}}(t) = \begin{cases} 
      0 & t=0 \\
      \max(0, Y_{\texttt{\textup{KL}}}(t-1) + Z_{\texttt{\textup{KL}}}(X_t)) & t > 0,
\end{cases}
\quad\quad
\tau_{\texttt{\textup{KL}}}^c \de \min \{t \in \mathbb{N} : Y_\texttt{\textup{KL}}(t) \geq c \}
\end{equation}

for some choice of threshold $c \geq 0$. 

\subsubsection{Score-Based Cumulative Sum (SCUSUM) Algorithm}

The Score-Based Cumulative Sum (SCUSUM) algorithm (\cite{wu_scusum_aistats}) replaces log-likelihood ratios by differences of Hyv\"arinen scores.
With instantaneous detection score $Z_{\texttt{\textup{F}}}(X)$ defined as in \eqref{eq:test_hscore}, the SCUSUM algorithm defines a cumulative detection score $Y_\texttt{F}(t)$ and a stopping rule $\tau_\texttt{F}^c$ as
\begin{equation}
\label{eq:scusum}
Y_{\texttt{\textup{F}}}(t) = \begin{cases} 
      0 & t=0 \\
      \max(0, Y_{\texttt{\textup{F}}}(t-1) + Z_{\texttt{\textup{F}}}(X_t)) & t > 0,
\end{cases}
\quad\quad
\tau_{\texttt{\textup{F}}}^c \de \min \{t \in \mathbb{N} : Y_\texttt{\textup{F}}(t) \geq c \}
\end{equation}

for some choice of stopping threshold $c > 0$.  
\begin{remark}
The presentation of SCUSUM in \cite{wu_scusum_aistats} includes a scaling factor $\lambda >0$ in the definition of $Z_{\texttt{\textup{F}}}(X)$.  For the limited purposes of this presentation, we can omit $\lambda$ without loss of generality.
\end{remark}

The pre-change drift  $\mathbb{E}_\infty[Z_\texttt{\textup{F}}(X)]$ and post-change drift $\mathbb{E}_1[Z_\texttt{\textup{F}}(X)]$ were calculated in \cite{wu_scusum_aistats}, where it was further demonstrated that the pre-change drift is strictly negative and the post-change drift strictly positive under certain conditions.  Furthermore, \cite{wu_scusum_aistats} provided theorems which bound the mean time to false alarm \eqref{eq:def_arl} and detection delay \eqref{eq:wadd} for any given choice of $P_\infty, P_1$.

\section{Diffusion-Based Detection}
\label{sec:diffusion_detection}

We have presented the Fisher divergence as well as a score-based test and stopping rule.  In this section, we first define the diffusion divergence, which generalizes the Fisher divergence. 
We next propose a diffusion-based test and stopping rule, generalizing the score-based algorithms introduced in Section~\ref{sec:hypo_cpd_background}. We conclude by providing theorems which bound the properties of the proposed diffusion-based test and algorithm.

Though the score-based methods are special cases of the diffusion-based ones (and though the diffusion-based methods still utilize the score $\nabla \log p(X)$), we shall refer to the Fisher divergence-based methods as \textit{score-based} and the methods utilizing $m(X)$ as \textit{diffusion-based} for clarity.

\subsection{Diffusion Divergence}
\label{sec:background_diffusion}
A generalization of the Fisher divergence was studied in  \cite{barp2019minimum} and applied to change-point detection in \cite{altamirano2023robust}.  This generalization, which we call the \textit{diffusion divergence}, transforms the difference between scores in Definition~\ref{def:fisher_divergence} by multiplication with some matrix-valued function $m(X) : \mathbb{R}^d \mapsto \mathbb{R}^{d \times w}$:
\begin{definition}[Diffusion Divergence]
\label{def:diffusion_divergence}
The diffusion divergence between distributions $P$ and $Q$ is defined to be:
\begin{equation*}
    \mathbb{D}_m(P \| Q) = \mathbb{E}_{X\sim P} \left[ \tfrac{1}{2} \left\| m^T(X) \big( \nabla_{{X}} \log p(X)- \nabla_{{X}} \log q(X)\big)\right \|_2^2 \right]
\end{equation*}
for some matrix-valued function $m(X) : \mathbb{R}^d \mapsto \mathbb{R}^{d \times w}$.
\end{definition}

We impose conditions upon $m(X)$ in Section~\ref{subsec:diffusion_based_detection}. 
 For ease of notation, we denote the \textit{function} $m(X)$ simply as $m$ in the subscript of $\mathbb{D}$.

We next define the \textit{diffusion-Hyv\"arinen score}, which generalizes the Hyv\"arinen score of Definition~\ref{def:hyvarinen_score}.
\begin{definition}[Diffusion-Hyv\"arinen Score]
\label{def:diffusion_hscore}
Denoting the divergence of a function $f : \mathbb{R}^d \mapsto \mathbb{R}^d$ with Jacobian $J_f$ as $\nabla \cdot f(X) = tr(J_f) = \sum_i {\partial f_i(X)}/{\partial X_i}$, we define
  \begin{equation*} % Return To Multline in Two Column Mode
    \mathcal{S}_m(X, P) \de \tfrac{1}{2} \left \|m^T(X) (\nabla_{X} \log p(X)) \right \|_2^2 
    \maybeNewline
    + \nabla \cdot m(X) m^T(X) \nabla \log p(X),
\end{equation*} % Return To Multline in Two Column Mode
\end{definition}
We conclude this section by presenting an important identity:
\begin{lemma}[Diffusion Drifts]
\label{lemma:drifts}
For distributions $P,Q$:
\begin{align}
\label{eq:drifts}
\mathbb{E}_P[\mathcal{S}_m(X, P) - \mathcal{S}_m(X, Q)] = -\mathbb{D}_m(P \| Q) ,\quad\quad \mathbb{E}_Q[\mathcal{S}_m(X, P) - \mathcal{S}_m(X, Q)] = \;\;\mathbb{D}_m(Q \| P).
\end{align}
\end{lemma}
\begin{proof}
The proof is given in Appendix \ref{appendix:proofs}.
\end{proof}

\subsection{Diffusion-based Detection}
\label{subsec:diffusion_based_detection}

We begin by imposing mild regularity conditions on $m(X)$, which mirror the regularity conditions presented in \cite{hyvarinen2005estimation}.  We shall assume Assumption \ref{assumption:hyvarinen_regularity} everywhere in this paper.

\begin{assumption}[Diffusion-Hyv\"arinen Regularity Conditions \cite{hyvarinen2005estimation}]
\label{assumption:hyvarinen_regularity}
    For $(R,S) \in \{(P_\infty,P_1),(P_1,P_\infty)\}$, we assume that:
    \begin{enumerate}
        \item $m(X)$ and $\nabla \log s(X)$ are differentiable with continuous derivatives,
        \item $\mathbb{E}_R[\| m^T(X) \nabla \log s(X)\|^2] < \infty$ and $\mathbb{E}_R[\| m^T(X) \nabla \log r(X)\|^2] < \infty$, and
        \label{item:finite}
        \item $s(X) m(X) m^T(X) \nabla \log r(X) \rightarrow 0$ as $\|X\| \rightarrow \infty$. \label{item:goes_to_zero_large_x}
    \end{enumerate}
\end{assumption}
\begin{lemma}[Finite Divergences]
\label{lemma:finite_diffusion_divergence}
Part~\ref{item:finite} of Assumption~\ref{assumption:hyvarinen_regularity}  implies that $\mathbb{D}_m(P_\infty \| P_1)< \infty$ and $\mathbb{D}_m(P_1 \| P_\infty) < \infty$.
\end{lemma}
\begin{proof}
    The proof is given in Appendix \ref{appendix:proofs}.
\end{proof}

As an expectation over a norm, the diffusion divergence is strictly nonnegative for all distributions.  We next provide an assumption that the diffusion divergence is nonzero, which we shall assume everywhere this paper:
\begin{assumption}[Positive Divergences]
\label{assumption:positive_diffusion_divergence}
    We assume that $m(X)$ is chosen such that    $\mathbb{D}_m(P_\infty \| P_1) > 0$ and $\mathbb{D}_m(P_1 \| P_\infty) > 0.$
\end{assumption}
\begin{remark}
\label{remark:m_invertible}
The work of \cite{barp2019minimum} uses the diffusion divergence to generalize score matching, which requires the diffusion divergence to satisfy $\mathbb{D}_m(P\|Q) = 0 \iff P = Q$ for \textit{all} $P,Q$.  For noninvertible $m(X)$, there might exist $P,Q$  for which $\nabla \log p(X) - \nabla \log q(X)$ lies in the kernel of $m(X)$ almost everywhere (causing $\mathbb{D}_m(P \|Q) = 0$ for $P \neq Q$), and hence \cite{barp2019minimum} imposed a condition of invertibility on $m(X)$ for all $X \in \mathbb{R}^d$.  
Throughout this paper, all diffusion divergences take arguments only from $\{ P_\infty, P_1 \}$.  Therefore, we require only Assumption~\ref{assumption:positive_diffusion_divergence} and can therefore consider $m(X)$ which is not invertible. Hence, Assumption~\ref{assumption:positive_diffusion_divergence} permits $m(X)$ to be non-square ($w \neq d$) and the theorems of this paper allow arbitrary $w$.
\end{remark}

Next, we generalize the test statistic of \cite{wu2022score} and the instantaneous detection score of \cite{wu_scusum_aistats} to include $m(X)$.  
\begin{definition}[Diffusion Test Statistic and Instantaneous Detection Score] 
\label{def:diffusion_inst_det_score}
For any choice of $m(X)$ that satisfies Assumption \ref{assumption:hyvarinen_regularity}, and with $\mathcal{S}_m(\cdot, P)$ following Definition \ref{def:diffusion_hscore}, we define
\begin{equation}
    Z_m(X) \de \mathcal{S}_m(X, P_\infty) - \mathcal{S}_m(X, P_1).
\end{equation}
\end{definition}
Note that $Z_m(\cdot)$ is a function of its subscript unlike $Z_{\texttt{\textup{KL}}}(\cdot)$ and $Z_{\texttt{\textup{F}}}(\cdot)$.  
We next propose a diffusion-based hypothesis test.
\begin{definition}[Diffusion Hypothesis Test] \label{definition:diffusion_test} 
With some fixed choice of stopping threshold $c \in \mathbb{R}$, we define
\begin{equation}
T_m^c (( X_i)_{i=1}^n) = \begin{cases}
      0 & \text{ if } \;\sum_{i=1}^n Z_m(X_i) < c \\ 1 & \text{ else} \end{cases} 
\end{equation}
\end{definition}
We further propose the diffusion change-point detection stopping rule.
\begin{definition}[Diffusion Change-Point Detection Stopping Rule]
\label{def:diffusion_stopping_rule}
With some fixed choice of stopping threshold $c > 0$, we define 
\begin{equation}
Y_m(t) = \begin{cases} 
      0 & t=0 \\
      \max(0, Y_m(t-1) + Z_{m}(X_t)) & t > 0
\end{cases}
\quad
\text{ and }
\quad
    \tau_m^c \de \min \{t \in \mathbb{N} : Y_m(t) \geq c \}
\end{equation}
\end{definition}

\subsection{Properties of the Proposed Hypothesis Test}
We next present a theorem which bounds the performance of the diffusion hypothesis test.

\begin{theorem}[Error Exponent]
\label{theorem:stein_diffusion}
    Fix $P_\infty, P_1$, and $\overline{\alpha} \in (0, 1)$.   Let $m(X) : \mathbb{R}^d \mapsto \mathbb{R}^{d \times w}$ be some diffusion matrix function which satisfies both Assumption \ref{assumption:hyvarinen_regularity} and 
    $\mathbb{E}_1 [\exp (-Z_m(X))] = 1$.
    Recall $T^c_m$ from Definition~\ref{definition:diffusion_test} and  $\mathcal{B}(\cdot)$ from  Definition~\ref{def:error_exponent}.
    
    Then, there exists some $c$ such that (i) $\lim_{n \rightarrow \infty} \alpha_n(T^c_m) \leq \overline{\alpha}$ and 
        (ii) $\mathcal{B}(T^c_m) \geq \mathbb{D}_m(P_\infty \| P_1)$.
\end{theorem}
\begin{proof}
    The proof is given in Appendix \ref{appendix:proofs}.
\end{proof}

\subsection{Properties of the Proposed Stopping Rule}

We next present theorems quantifying the mean time to false alarm and detection delay of the diffusion change-point detection stopping rule, which are generalizations of corresponding results for the Fisher divergence presented in  \cite{wu_scusum_aistats}.

\begin{theorem}[Average Run Length]
\label{theorem:arl}
For $\tau_m^c$ following Definition~\ref{def:diffusion_stopping_rule}, if $m(X)$ satisfies
$\mathbb{E}_\infty [\exp Z_m(X)] \leq 1$,
then $\mathbb{E}_\infty [\tau_m^c] \geq e^{c}$.
\end{theorem}
\begin{proof}
    The proof is given in Appendix \ref{appendix:proofs}.
\end{proof}

\begin{theorem}[Detection Delay]
\label{theorem:edd}
For $\tau_m^c$ following Definition~\ref{def:diffusion_stopping_rule}, the worst-case average detection delay is given by
\begin{equation}
\label{eq:diffusion_edd}
\quad   \mathcal{L}_{\texttt{\textup{WADD}}}(\tau_m^c) \sim \frac{c}{\mathbb{D}_m(P_1 \| P_\infty)} \text{ as } c \rightarrow \infty,
\end{equation}
where $f(c) \sim g(c)$ as $c \rightarrow \infty$ means that $\lim_{c \rightarrow \infty } {f(c)}/{g(c)} = 1$.
\end{theorem}
\begin{proof}
The proof is given in Appendix \ref{appendix:proofs}.
\end{proof}

For each of our detection tasks, we have generalized one unique score-based detection algorithm into a collection of diffusion-based algorithms, each of which is different in both its choice of $m(X)$ and in its performance.  We next turn our attention to making a good choice of diffusion matrix function so that our detection algorithms can perform well.

\section{Optimization of the Diffusion-Matrix Function}
\label{sec:optimization}
In this section, we propose a method of finding the optimal $m(X)$.
\subsection{Objective Function: Change-Point Detection}
\label{sec:optimization_objective_cpd}
Under the objective presented in \eqref{eq:cpd_objective}, we wish to minimize the detection delay \eqref{eq:wadd} subject to a constraint on the average run length \eqref{eq:def_arl}.  Denoting our constraint on the average run length as $\overline{\gamma}$, we can set a stopping threshold to $c = \log \overline{\gamma}$; by Theorem~\ref{theorem:arl} the constraint of \eqref{eq:cpd_objective} (that ARL$(\tau) \geq \overline{\gamma}$) will be satisfied so long as $\mathbb{E}_\infty[\exp Z_m(X)] \leq 1$.  

The expected detection delay is a function of stopping threshold and $\mathbb{D}_m(P_1 \| P_\infty)$.  Since we have already set the stopping threshold $c = \log \overline{\gamma}$, we minimize \eqref{eq:diffusion_edd} by maximizing $\mathbb{D}_m(P_1 \| P_\infty)$.  This produces the objective: 
\begin{equation}
\label{eq:cpd_lagrangian}
    \max_m \;\mathbb{D}_m(P_1 \| P_\infty)\; \text{ s.t. } \mathbb{E}_\infty[\exp Z_m(X)] = 1.
\end{equation}
For any $m(X)$ satisfying Assumption~\ref{assumption:hyvarinen_regularity}, and any $k > 0$, it is easy to see that $Z_{km}(X) = k^2 Z_m(X)$ for all $X \in \mathbb{R}^d$ and that $\mathbb{D}_{km}(P_1 \| P_\infty) = k^2 \mathbb{D}_m (P_1 \| P_\infty)$.  Thus, if we were to optimize by simply maximizing $\mathbb{D}_m (P_1 \| P_\infty)$, then for any $m(X)$, $\mathbb{D}_{km}(P_1 \| P_\infty) > \mathbb{D}_m(P_1 \| P_\infty)$ whenever $k > 1$.  Thus, some constraint is necessary for a maximum of \eqref{eq:cpd_lagrangian} to exist.  Furthermore, in change-point detection (hypothesis testing), a stopping rule (test) using instantaneous detection score (statistic) $Z_m(\cdot)$ and threshold $c$ is identical to another stopping rule (test) using instantaneous detection score (statistic) $k^2Z_m(\cdot)$ and threshold $k^2c$ -- adjusting the scale $k$ does not actually change the algorithm (scale-invariance over the densities of  $P_\infty, P_1$ is an advantage of score-based methods but scale-invariance over $m$ frustrates this optimization).  Optimization over only $\mathbb{D}_m(P_1 \| P_\infty)$ may simply lead to a larger scaling on $m(X)$, and this optimization would neither improve the test nor converge to any limit.
Thus, some constraint on the scale of $m(X)$ is necessary, and the condition $\mathbb{E}_\infty[\exp Z_m(X)] \leq 1$ plays this role. Though the constraint was introduced as a proof technique, it is in fact essential to the optimization process.

\subsection{Objective Function: Hypothesis Testing}
\label{sec:optimization_objective_hypo}

We wish to find a diffusion hypothesis test between $P_\infty, P_1$ which performs reasonably well \textit{for all batch sizes $n$}.  The popular objective function of \eqref{eq:objective_hypothesis_testing}, however, is explicitly a function of a particular batch size $n$.  We modify the objective of \eqref{eq:objective_hypothesis_testing} to better account for the variable nature of the batch size:
\begin{equation}
\label{eq:objective_hypo_asymptotic}
    \max_T \; \mathcal{B}(T) \; \text{ subject to } \; \lim_{n \rightarrow \infty} \alpha_n (T) \leq \overline{\alpha},
\end{equation}
where $\mathcal{B}(\cdot)$ is defined by Definition~\ref{def:error_exponent}.  Though the objective of \eqref{eq:objective_hypo_asymptotic} considers the performance of a test for large $n$, it does not depend upon any one particular choice of $n$.

For any $m(X)$ following Assumption~\ref{assumption:hyvarinen_regularity} and satisfying $\mathbb{E}_1[\exp (-Z_m(X))] \leq 1$,  Theorem~\ref{theorem:stein_diffusion} provides that $\mathbb{D}_m(P_\infty \| P_1)$ is a lower-bound on the error exponent of some particular $T_m^c$ which obeys $\textit{any constraint}$ on asymptotic type I error probability.  
For any choice of $\overline{\alpha}$, we maximize the type II error exponent of a test by maximizing its lower bound, $\mathbb{D}_m(P_\infty \| P_1)$, while enforcing  $\mathbb{E}_1[\exp (-Z_m(X))]=1$ to satisfy the condition of the Theorem \ref{theorem:stein_diffusion}.

Thus, we propose that the following objective optimizes $m(X)$ for the problem of hypothesis testing:
\begin{equation}
\label{eq:hypo_lagrangian}
    \max_m \;\mathbb{D}_m(P_\infty \| P_1)\; \text{ s.t. } \mathbb{E}_1[\exp (-Z_m(X))] = 1.
\end{equation}

\begin{remark}
    Recalling Definition~\ref{def:diffusion_inst_det_score}, we observe that the objective of \eqref{eq:cpd_lagrangian} takes the form of   \eqref{eq:hypo_lagrangian} when the roles of $P_\infty, P_1$ are interchanged.
\end{remark}

\subsection{Optimality of Diffusion-Based Hypothesis Testing and Change-Point Detection}
\label{sec:optimization_analytical}

We first demonstrate that for a special case of two Gaussian distributions with a common covariance matrix,  $Z_{\texttt{\textup{KL}}}(\cdot) = Z_m(\cdot)$ for a particular choice of $m(X)$:  

\begin{theorem}
\label{theorem:gaussian_analytical}
    Consider vectors $\mu_\infty, \mu_1 \in \mathbb{R}^d$ and positive definite matrix $V \in \mathbb{R}^{d \times d}$.  Let $P_\infty = \mathcal{N}(\mu_\infty, V)$ and  $P_1 = \mathcal{N}(\mu_1, V)$, and define $M^* \de V^{\frac{1}{2}}$.  
    Then,
    $
        \forall X \in \mathbb{R}^d,\;Z_{M^*}(X) = Z_{\texttt{\textup{KL}}}(X)
    $.
\end{theorem}
\begin{proof}
    The proof is given in Appendix \ref{appendix:proofs}.
\end{proof}

In general, we do not expect any particular choice of $m(X)$ to exactly recover the LLR-based tests.  We demonstrate using another pair of Gaussians that there does not always exist an $m(X)$ such that $Z_m(\cdot) = Z_{\texttt{\textup{KL}}}(\cdot)$ with probability one.

\begin{theorem}
\label{theorem:ode}
There exists a pair of distributions $P_\infty, P_1$ such that for all $m(X)$,
$\mathbb{P}_\infty[Z_m(X) \neq Z_{\texttt{\textup{KL}}}(X)] > 0$ and $\mathbb{P}_1[Z_m(X) \neq Z_{\texttt{\textup{KL}}}(X)] > 0.$
%for at least one $X \in \mathbb{R}^d$.
\end{theorem}
\begin{proof}
    The proof is given in Appendix  \ref{appendix:proofs}.
\end{proof}
\begin{remark}
\label{remark:fisher_acheivable}
    For any $P_\infty, P_1$, $Z_m(\cdot) = Z_{\texttt{\textup{F}}}(\cdot)$ if $m(X) = I$.
\end{remark}

It is possible to bound the performance of the \textit{best} diffusion-based algorithm -- the algorithm corresponding to the \textit{best} choice of $m(X)$ -- in terms of the performance of the score-based and LLR-based algorithms.  By Remark~\ref{remark:fisher_acheivable}, we observe that the best diffusion-based algorithms can perform no worse than their score-based peers (though diffusion algorithms corresponding to a poorly-chosen $m(X)$ can perform arbitrarily poorly). The diffusion-based algorithms can never outperform their LLR-based counterparts, however, as the latter are  provably optimal under \eqref{eq:objective_hypothesis_testing} and \eqref{eq:cpd_objective}.

While we know that CUSUM \eqref{eq:cusum} is optimal under the metric of \eqref{eq:cpd_objective}, in general there can be stopping rules distinct from CUSUM which match the performance 
(\cite{moustakides_private_comms, tartakovsky_personal_communications})
of CUSUM. Therefore, while for some choices of $P_\infty, P_1$ we do not know whether LLR-like performance is attainable  (such as  for the $P_\infty, P_1$ of Theorem~\ref{theorem:ode}), we do know that this optimal performance is attainable for other choices of $P_\infty, P_1$ (such as those of Theorem~\ref{theorem:gaussian_analytical}).

In general, and especially for high-dimensional $P_\infty, P_1$, the best possible $m(X)$ will not be analytically identifiable.  In the general case, we simply wish to find an $m(X)$ that approximates a solution to \eqref{eq:cpd_lagrangian}, \eqref{eq:hypo_lagrangian}.  In the next section, we turn to machine learning to help with this approximation.

\subsection{Numerical Optimization}
\label{sec:optimization_numerical}

We propose a multi-layer perceptron (MLP) network $m : \mathbb{R}^d \mapsto \mathbb{R}^{d \times w}$ to approximate the best possible matrix-valued function for a given $P_\infty, P_1$.  
The MLP $m(X)$ can be trained to approximate a solution of \eqref{eq:cpd_lagrangian} or \eqref{eq:hypo_lagrangian}.

While we wish to optimize the MLP via gradient descent, the method of gradient descent requires differentiability of the loss function.  The constraints ($\mathbb{E}_\infty [\exp Z_m(X)]=1$ of \eqref{eq:cpd_lagrangian} and $\mathbb{E}_1[\exp (-Z_m(X))]=1$ of  \eqref{eq:hypo_lagrangian}) preclude gradient descent from being applied directly to these objectives.  We turn to the use of regularization to circumvent this issue.  
We propose the loss functions $\mathcal{L}_{\texttt{\textup{CPD}}}$ and $\mathcal{L}_{\texttt{\textup{HT}}}$ to numerically optimize $m(X)$ for change-point detection and hypothesis testing, respectively:
\begin{align}
\label{eq:cpd_loss_function}
\mathcal{L}_{\texttt{\textup{CPD}}}(m) &= -\mathbb{D}_m(P_1 \| P_\infty) + \kappa \big[ \log \mathbb{E}_\infty (\exp  (Z_m(X))) \big]^2, \\
\mathcal{L}_{\texttt{\textup{HT}}}(m) &= -\mathbb{D}_m(P_\infty \| P_1) + \kappa \big[ \log \mathbb{E}_1 (\exp (- Z_m(X))) \big]^2,
\end{align}
where $\kappa > 0$ is a training hyperparameter.  
The condition $\mathbb{E}_\infty[\exp Z_m(X)] = 1$ of \eqref{eq:cpd_lagrangian} ($\mathbb{E}_1[\exp(-Z_m(X))]=1$ of \eqref{eq:hypo_lagrangian}) is satisfied if and only if $[\log \mathbb{E}_\infty(\exp Z_m(X))]^2 = 0$ ($[\log \mathbb{E}_1(\exp (-Z_m(X)))]^2 = 0$).  Therefore, though it does not strictly enforce the exact conditions, gradient descent over $\mathcal{L}_{\texttt{\textup{CPD}}}(m)$ or $\mathcal{L}_{\texttt{\textup{HT}}}(m)$ causes the neural network $m(X)$ to approximately satisfy the conditions of  \eqref{eq:cpd_lagrangian} or \eqref{eq:hypo_lagrangian}.

\section{Numerical Simulations}
\label{sec:simulations}

\begin{figure*}               % span both columns
  \centering
  \subfloat[Gaussian Distributions]{\includegraphics[width=0.3\linewidth]{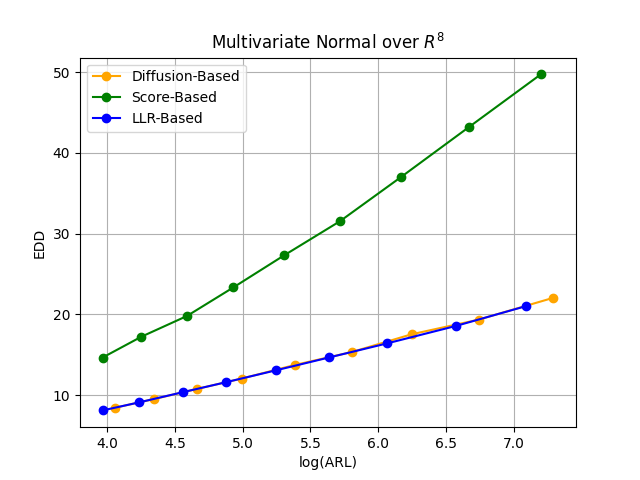}\label{fig:cpd_gauss}}
  \subfloat[GB-RBM Distributions]{\includegraphics[width=0.3\linewidth]{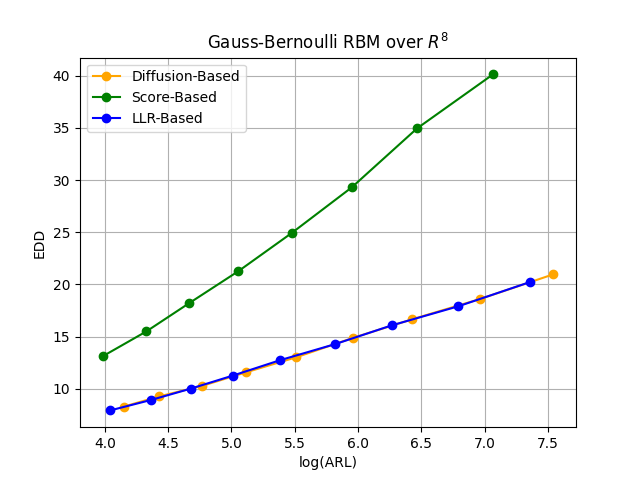}}\label{fig:cpd_trbm}
  \subfloat[Quartic Exponential Distributions]{\includegraphics[width=0.3\linewidth]{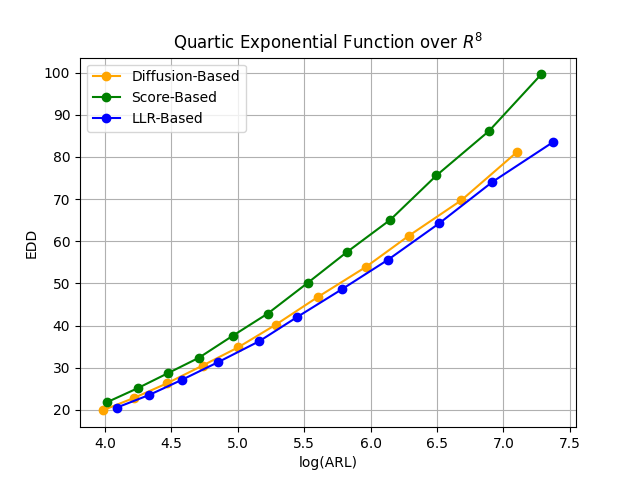}\label{fig:cpd_quar}} 
\caption{Performance of LLR-based, Fisher-based, and diffusion-based Change-Point Detection Stopping Rules.  In this simulation, $P_\infty, P_1$ were chosen to be (a) Gaussian Distributions, (b) Gauss-Bernoulli Restricted Boltzmann Machine Distributions, and (c) Quartic Exponential Distributions.  ARL and EDD are averaged over 10,000 sample paths.}
\label{fig:arl_edd}
\end{figure*}

\newcommand{\imgwid}{0.3\linewidth}

\newcommand{\rowlab}[1]{%
  \adjustbox{valign=c}{\rotatebox[origin=c]{90}{\small #1}}%
}

\newcommand{\rocimg}[1]{%
  \adjustbox{valign=c}{\includegraphics[width=.315\textwidth,
  trim=0 5bp 0 0,
  clip
  ]{#1}}%
}

\begin{figure*}[!t]
\centering
\setlength{\tabcolsep}{1pt}
\renewcommand{\arraystretch}{0}

\begin{tabular}{@{}c c c c@{}}

\rowlab{Gaussian} &
\rocimg{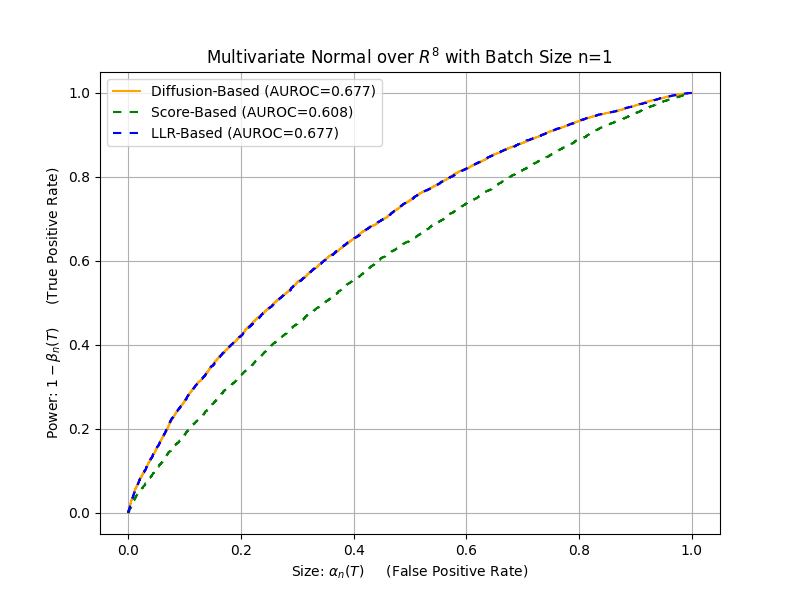} &
\rocimg{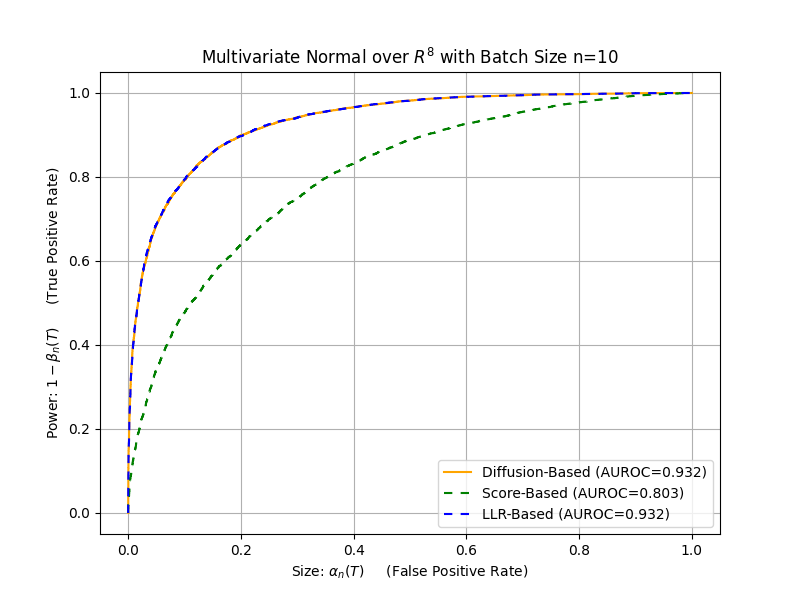} &
\rocimg{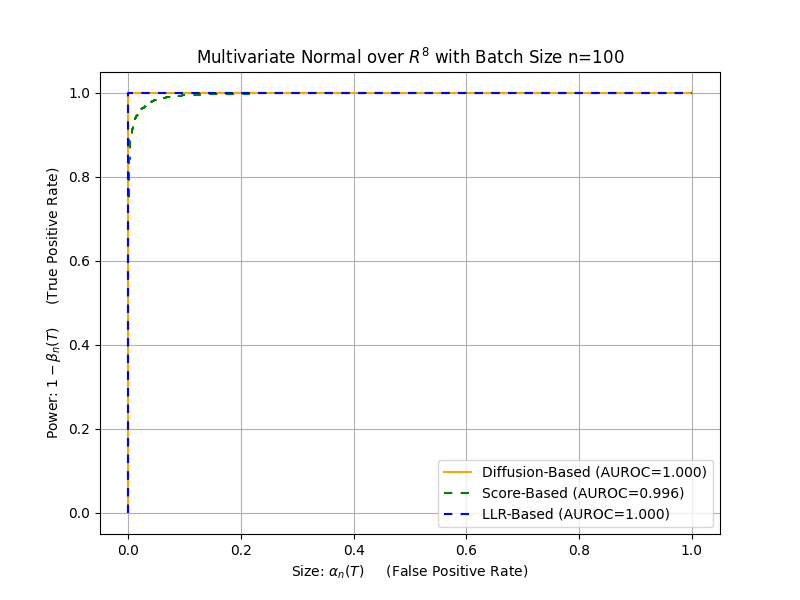} \\[-0.6em]

\rowlab{GB-RBM} &
\rocimg{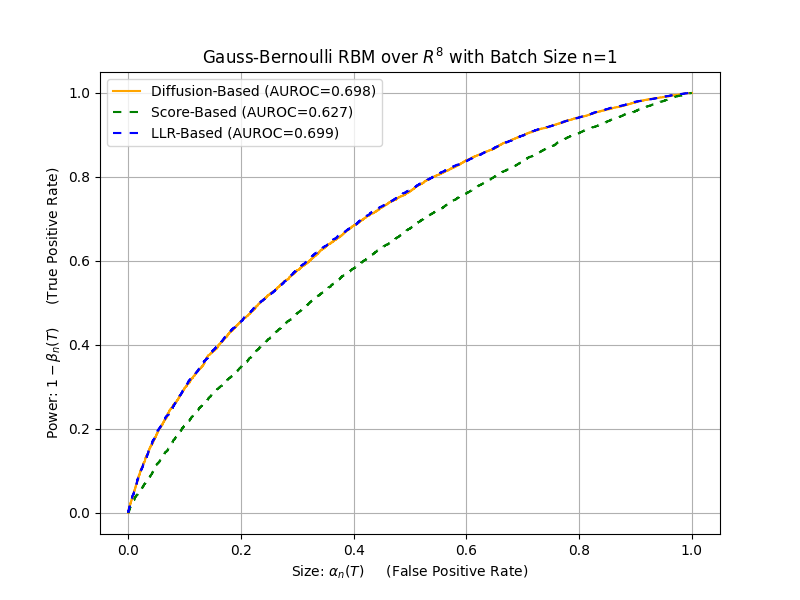} &
\rocimg{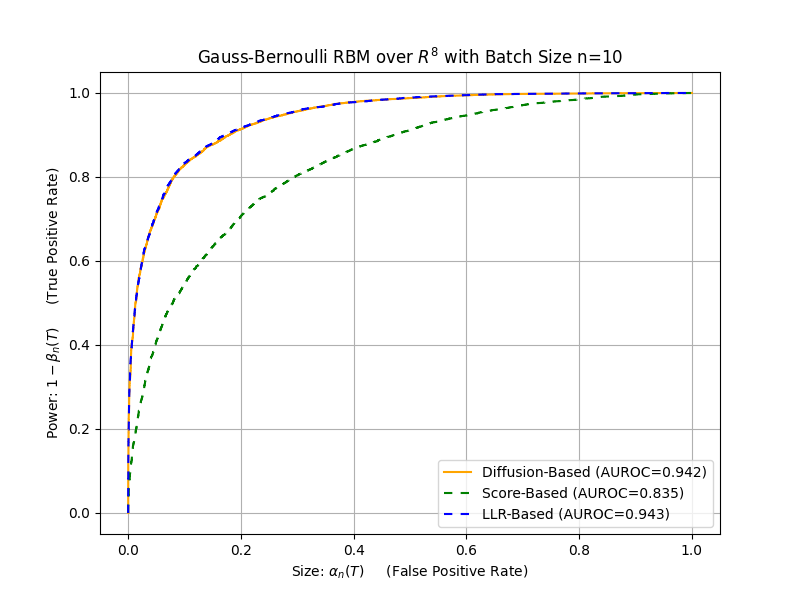} &
\rocimg{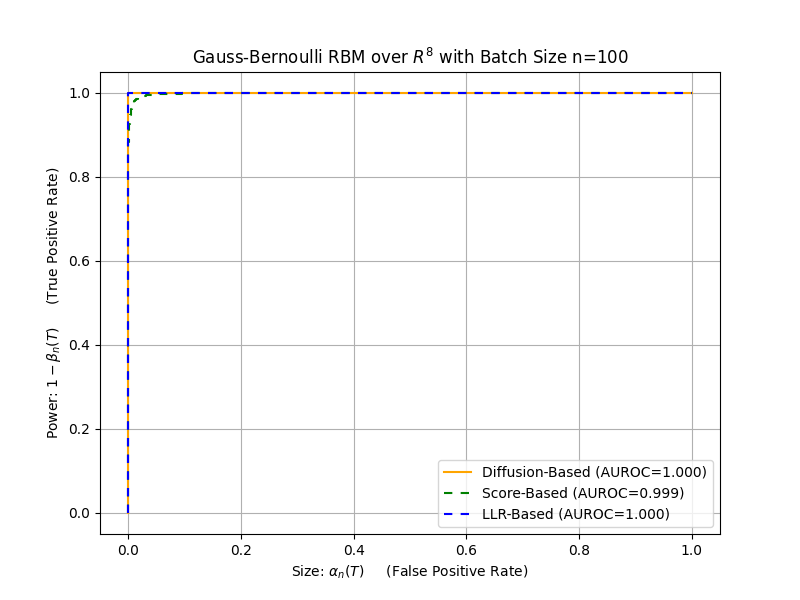} \\[-0.6em]

\rowlab{Quartic Exponential} &
\rocimg{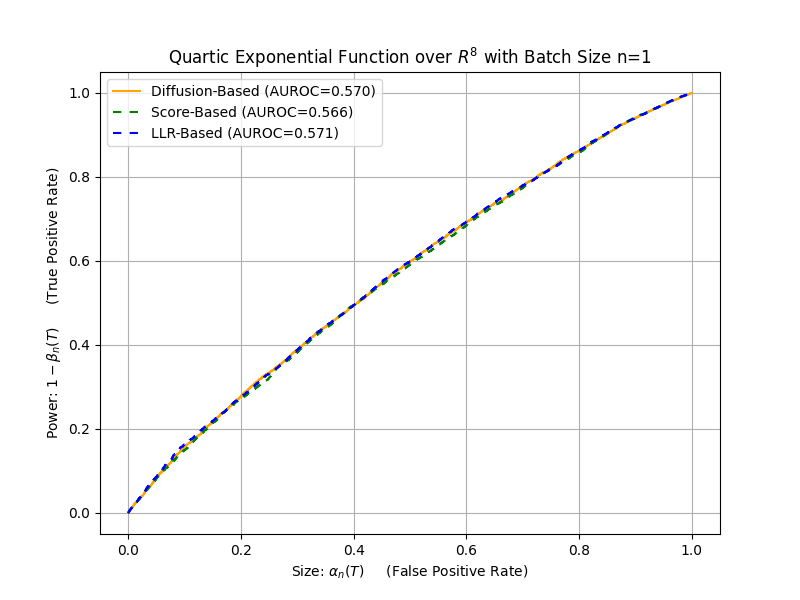} &
\rocimg{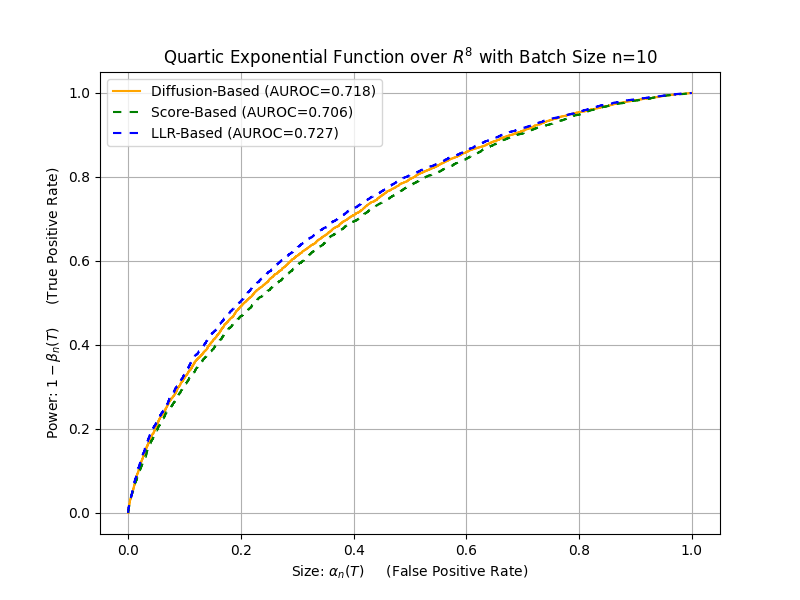} &
\rocimg{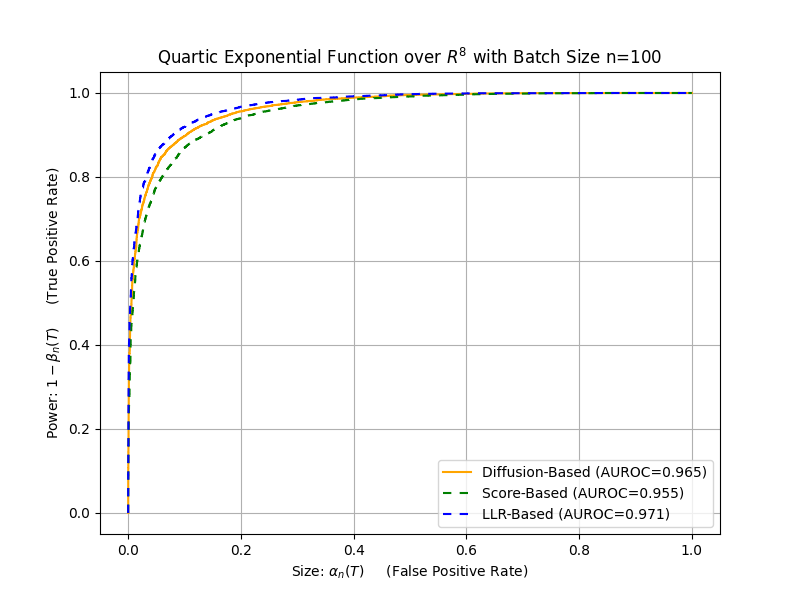} \\[0.35em]

& \small $n=1$ & \small $n=10$ & \small $n=100$

\end{tabular}

\vspace{-0.5em}
\caption{ROC curves for diffusion-based, score-based, and LLR-based hypothesis tests, plotted for batch sizes $n \in \{ 1, 10, 100 \}$.  $P_\infty, P_1$ are chosen to follow Gaussian Distributions (top), Gauss-Bernoulli Restricted Boltzmann Machine Distributions (middle), and Quartic Exponential Distributions (bottom).  Each ROC curve uses 10,000 batches.  Similar plots for $n \in \{5, 25, 50\}$ are shown in Figure~\ref{fig:roc_extra} in the Appendix.}
\label{fig:roc}
\end{figure*}

In this section, we implement and simulate the diffusion-based methods and compare their performances against those of their LLR-based and score-based analogs for both hypothesis testing and change-point detection.  
We begin by introducing three model classes with analytical score functions.

\subsection{Model Classes}

In this section, we shall use the notation of $p(X)$ to refer to the density of some distribution $P$, and shall use the notation of $\tilde{p}$ to refer to the \textit{unnormalized} density of some distribution $P$.  Two of our three model classes provide only an unnormalized density.

For each model class below, the simulations are performed with parameters chosen in light of \cite[Proposition 5]{wu_scusum_aistats} to demonstrate cases where score-based methods noticeably underperform LLR-based methods.  We provide details and parameters for our model classes in Appendix~\ref{appendix:details_sims}.

\subsubsection{Gaussian Distribution}
The multivariate Gaussian distribution is parameterized by a vector $\mu \in \mathbb{R}^d$ and a positive-definite matrix $\Sigma \in \mathbb{R}^{d \times d}$. Its density and score function are given by:
\begin{equation}
    p(X) = \frac{1}{\sqrt{(2\pi)^d |\Sigma|}}\exp \big(-\tfrac{1}{2}(X-\mu)^T \Sigma^{-1} (X-\mu) \big), \quad
    \nabla \log p(X) = -\Sigma^{-1}(X-\mu).
\end{equation}
\subsubsection{Gauss-Bernoulli Restricted Boltzmann Machine}

The Gauss-Bernoulli Restricted Boltzmann Machine (GB-RBM) (\cite{gbrbm_without_tears}) is simultaneously an energy-based model and a score-based model.  Thus, it is a natural candidate for comparisons between LLR-based and score-based algorithms.

As presented in \cite[Equation 3]{gbrbm_without_tears}, the conditional density of visible variable conditioned upon latent variable is a Gaussian distribution with diagonal covariance; here, we relax this condition and require only that it be positive definite.  For this formulation of the GB-RBM with latent dimension $h$, the model is parameterized by vectors $\mu \in \mathbb{R}^d, \phi \in \mathbb{R}^h, W \in \mathbb{R}^{d \times h}$, and positive definite matrix $\Sigma \in \mathbb{R}^{d \times d}$.

% We first provide the free-energy of the GB-RBM
We provide the free-energy $F$, unnormalized density $\tilde{p}$, and score $\nabla \log \tilde{p}$ of the GB-RBM:
\begin{gather}
\tilde{p}(X) = \exp(-F(X)),\quad 
 \nabla \log \tilde{p} (X)  = -\Sigma^{-1} (X-\mu) 
    \maybeNewline
    + \Sigma^{-1} W \text{ Sigmoid}(\phi + W^T \Sigma^{-1} X) \\
   \text{ where } F(X) = \tfrac{1}{2} (X-\mu)^T \Sigma^{-1}(X-\mu) 
    \maybeNewline
    - \mathbf{1}^T \text{Softplus}(\phi + W^T \Sigma^{-1}X), \notag 
\end{gather}

\subsubsection{Quartic Exponential Model}
We finally consider an exponential model class that is quartic (fourth-order) in $X$.  This model is parameterized by $\mu \in \mathbb{R}^d$ and positive-definite matrix $\Sigma \in \mathbb{R}^{d \times d}$.  The unnormalized density and score function of this model are:
\begin{equation}
\label{eq:quartic_density}
    \tilde{p}(X) = \exp\big( -(X^{\odot 2} - \mu)^T \Sigma^{-1}(X^{\odot 2} - \mu) \big), \quad
    \nabla \log \tilde{p}(X) = -4 \Sigma^{-1} (X^{\odot 2}-\mu) \odot X,
\end{equation}
where $\odot$ denotes element-wise multiplication and $X^{\odot 2}$ denotes element-wise squaring.

\subsection{Implementation of Algorithms}

\subsubsection{LLR-Based Algorithms}
For model classes that provide normalized densities, we implement the LLR-based algorithms following their description in Section~\ref{sec:hypo_cpd_background}.  For model classes involving unnormalized densities, we must modify the test statistic and instantaneous detection score $Z_{\texttt{\textup{KL}}}(\cdot)$ to account for the lack of normalization.  

We let $\mathcal{R}^n_{\tilde{P}_\infty, \tilde{P}_1}$ be a Monte-Carlo estimate of $\mathbb{E}_1[\tilde{p}_\infty(x) / \tilde{p}_1(x)]$ using $n$ samples (Section~\ref{sec:partition_ratio_estimation}).  Define
\begin{equation}
\hat{Z}_{\texttt{\textup{KL}}}(X) \de \log \frac{\tilde{p}_1(X)}{\tilde{p}_\infty(X)} + \log \mathcal{R}_{\tilde{P}_\infty, \tilde{P}_1}^n.
\end{equation}
We let $\hat{Z}_\texttt{\textup{KL}}$ play the role of $Z_{\texttt{\textup{KL}}}(X)$ in our simulations.  We note that the estimation of $\mathcal{R}_{\tilde{P}_\infty, \tilde{P}_1}^n$ requires a significant quantity $n$ of data and that this estimation becomes generally intractable as the data dimension becomes large.  This approximation is discussed further in Section~\ref{sec:partition_ratio_estimation}.

\subsubsection{Score-Based Algorithms}
Score-based algorithms follow the implementation given in Section~\ref{sec:hypo_cpd_background}.  We note that for any unnormalized density $\tilde{p}$ with normalizing constant $Z$, the score-based algorithms are invariant to the scale of $\tilde{p}_\infty(\cdot), \tilde{p}_1(\cdot)$:
\begin{align*}
    \nabla \log p(X) &= \nabla \log ({\tilde{p}(X)}/{Z}) 
    = \nabla \log \tilde{p}(X) - \nabla \log Z = \nabla \log \tilde{p}(X) \eqnum
\end{align*}

\subsubsection{Diffusion-Based Algorithms}
Like the score-based algorithms, diffusion-based algorithms are also scale invariant; hence, the diffusion-based algorithms do not need to account for scaling constants $C_\infty, C_1$.

We simulate square matrix-valued functions.  We let a feed-forward neural network play the role of $m(X)$, mapping data in $\mathbb{R}^d$ to matrices in $\mathbb{R}^{d \times d}$.  We optimize this choice of $m(X)$ following the loss function of \eqref{eq:cpd_loss_function}.  Details regarding the architecture of this neural network and its training are provided in Appendix~\ref{appendix:details_sims}.

\subsection{Results}

For both hypothesis testing and change-point detection, and for each of the three proposed model classes, we compare the performances of the diffusion-based, score-based, and LLR-based algorithms.

% In Figures~\ref{fig:roc_gauss}, \ref{fig:roc_trbm}, and \ref{fig:roc_quar},
In Figure~\ref{fig:roc}, we compare the performances of each method via Receiver Operating Characteristic (ROC) curves, considering the performance of each test under all possible choices of threshold $c$.  We observe that for the Gaussian and GB-RBM  distributions, the performance of the diffusion-based hypothesis tests nearly match the performance of their LLR-based counterparts, and that for the quartic exponential distribution, the diffusion-based hypothesis test outperforms the score-based test.

In Figure~\ref{fig:arl_edd}, we compare the performance of each change-point detection stopping rule.  Each algorithm is run with many choices of threshold $c$.  We recall the average run length (ARL) from \eqref{eq:def_arl}, and we define the expected detection delay (EDD) of a stopping rule $\tau$ to be $\text{EDD}(\tau) = \mathbb{E}_1[\tau].$
We plot the ARL and EDD for the LLR-based, score-based, and diffusion-based stopping rules.  Each curve illustrates the ARL and EDD of an algorithm for several choices of stopping threshold $c$.
Again, we observe that the diffusion-based stopping rule matches the LLR-based one in performance for the Gaussian and GB-RBM simulations and outperforms the score-based algorithm for the simulation involving the quartic exponential distribution.

Across all three model classes and for each of our two detection tasks, we observe that the diffusion-based algorithms perform no worse than the score-based methods and no better than the LLR-based methods, and that they occasionally match the LLR-based methods in performance.

\section{Conclusion}
In this paper, we have proposed a hypothesis test and a change-point detection stopping rule which utilize the diffusion divergence.  We have studied the properties of these methods, calculating and bounding their performance metrics.  We have developed an objective over the diffusion matrix function for hypothesis testing and change-point detection applications.  
We have demonstrated that the performances of the best-possible diffusion algorithms are no worse than the performances of the score-based algorithms and no better than the performances of the LLR algorithms.  
We  have proposed a loss function for the training of $m(X)$ and demonstrated the stability of this process by way of simulation.  Finally, we have demonstrated the implementability of the diffusion-based algorithms in numerical simulations.

%\iffalse
\section*{Acknowledgments}
This work was supported in part by the U.S. National Science Foundation under awards numbers 2334898 and 2334897.
The authors thank Prof. George V. Moustakides and Prof. Alexander G. Tartakovsky for insightful discussions.
%\fi

\bibliographystyle{tmlr}
\bibliography{bib}

\FloatBarrier
\newpage

\appendix
{\Huge \textbf{Appendix}} \\

\section{Details of Numerical Simulations}
\label{appendix:details_sims}

\subsection{Model Class Details}
We begin by defining the model classes for which simulation results are provided.  We provide $V^*, \mu_1^*$, and $\mu_\infty^*$ in \eqref{eq:param_V}, \eqref{eq:param_mu_1}, and \eqref{eq:param_mu_infty}, respectively.

\subsubsection{Gaussian Distribution}
We let $P_\infty = \mathcal{N}(\mu_\infty^*, V^*)$ and $P_1 = \mathcal{N}(\mu_1^*, V^*)$.
\subsubsection{Gauss-Bernoulli Restricted Boltzmann Machine}

We first create $W_\infty \in \mathbb{R}^{8 \times 6}$ where each element of $W_\infty$ is drawn i.i.d. from a $\mathcal{N}(0,1)$ distribution, and generate $\phi_\infty \in \mathbb{R}^6$ in the same way. 
We then sample $W_+ \in \mathbb{R}^{8 \times 6}$ and $\phi_+ \in \mathbb{R}^6$ such that each element is drawn i.i.d. from $\mathcal{N}(0, 0.1^2)$.  We then let $W_1 = W_\infty + W_+$ and $\phi_1 = \phi_\infty + \phi_+$.

We let $P_\infty$ be the GB-RBM with $\mu = \mu_\infty^*, \Sigma = V^*, W=W_\infty$ and $\phi = \phi_\infty$.  We then let $P_1$ be the GB-RBM with $\mu=\mu_1^*, \Sigma=V^*, W=W_1$, and $\phi=\phi_1$.

\subsubsection{Quartic Exponential Distribution}
We let $P_\infty$ be parameterized by $\mu = \mu_\infty^*$ and $\Sigma = V^*$.  We let $P_1 $ be parameterized by $\mu = \mu_1^*$ and $\Sigma = V^*$.

We note that if $\mu = 0$, the density of \eqref{eq:quartic_density} can be rewritten as
\begin{equation}
    \tilde{p}(X) = \exp\bigg(-\sum_i \sum_j (\Sigma^{-1})_{i,j} X_i^2 X_j^2 \bigg).
\end{equation}

\subsection{Sampling}
We sample from $P_\infty, P_1$ to create a training dataset of $100,000$ samples and a test dataset of $10,000$ samples.  We sample from all distributions via the Metropolis-Hastings algorithm (\cite{roberts_rosenthal}) except for the Gaussian distribution, where we perform direct Gaussian sampling.  We sample new data for the creation of the curves of Figures~\ref{fig:arl_edd} 
and \ref{fig:roc}.
% \ref{fig:roc_gauss}, \ref{fig:roc_trbm}, and \ref{fig:roc_quar}.

\subsection{Partition Ratio Estimation}
\label{sec:partition_ratio_estimation}

Define
\begin{equation}
    C_\infty \de \int_{\mathbb{R}^d} \tilde{p}_\infty(X) dX \quad \text{ and } \quad C_1 \de \int_{\mathbb{R}^d} \tilde{p}_1(X) dX.
\end{equation}
Then, the LLR-based test statistic and instantaneous detection score is
\begin{align*}
    Z_{\texttt{\textup{KL}}}(X) &= \log\frac{p_1(X)}{p_\infty(X)} = \log\frac{\tilde{p}_1(X)}{\tilde{p}_\infty(X)}\frac{C_\infty}{C_1}  = \log\frac{\tilde{p}_1(X)}{\tilde{p}_\infty(X)} + \log \frac{C_\infty}{C_1}, \eqnum
\end{align*}
which presents a problem, as $C_\infty / C_1$ is not known.  
%It is, however, equivalent to an expectation
It can, however, be calculated from an expectation:
\begin{align*}
    \mathbb{E}_1 \bigg[\frac{\tilde{p}_\infty(X)}{\tilde{p}_1(X)}\bigg] 
    &= \mathbb{E}_1\bigg[\frac{{p}_\infty(X
    )}{{p}_1(X)} \frac{C_\infty}{C_1}\bigg] =\int_{\mathbb{R}^d} p_1(X)\frac{p_\infty(X)}{{p}_1(X)} \frac{C_\infty}{C_1} dX =\frac{C_\infty}{C_1} \int p_\infty(X) dX = \frac{C_\infty}{C_1},
\end{align*}
which we can approximate with a sample mean:
\begin{equation}
    \mathcal{R}_{\tilde{P}_\infty, \tilde{P}_1}^n \de \frac{1}{n} \sum_{i=1}^n \frac{\tilde{p}_\infty(X_i)}{\tilde{p}_1(X_i)}
\end{equation}
for a dataset $(X_i)_{i=1}^n$ of i.i.d. samples of $P_1$.  
Thus, for our LLR-based tests, we let 
\begin{equation}
\hat{Z}_{\texttt{\textup{KL}}}(X) \de \log \frac{\tilde{p}_1(X)}{\tilde{p}_\infty(X)} + \log \mathcal{R}_{\tilde{P}_\infty, \tilde{P}_1}^n
\end{equation}
play the role of $Z_\texttt{\textup{KL}}(\cdot)$.

\subsection{Training}

We create a feed-forward neural network with 8-dimensional input, a single 36-dimensional hidden layer, and a 64-dimensional output.  We place a Sigmoid activation function after all non-output layers.  For each input $X \in \mathbb{R}^8$, we reshape the network's output $\overline{m}(X) \in \mathbb{R}^{64}$ into a matrix  $m(X) \in \mathbb{R}^{8 \times 8}$ and multiply this output by a constant $0.1$, which has been found to sometimes improve the stability of the training process during the first few epochs.

Training was performed via Adam (\cite{adam_optimizer}) with learning rates $0.035$ (Quartic Exponential Distribution), $0.04$ (Gaussian Distribution), and $0.01$ (Gauss-Bernoulli RBM).  Across all distributions, training was performed with L2-regularization of $1\cdot 10^{-5}$ and $\kappa = 10$.  
Hyperparameters were tuned via inspection of loss history; it should be noted that the competing models (Fisher-Based test, LLR-Based test) have no tunable parameters.

\subsection{Additional Plots}
We present additional ROC curves in Figure~\ref{fig:roc_extra}.

\begin{figure*}[t]
\centering
\setlength{\tabcolsep}{1pt}
\renewcommand{\arraystretch}{0}

\begin{tabular}{@{}c c c c@{}}

\rowlab{Gaussian} &
\rocimg{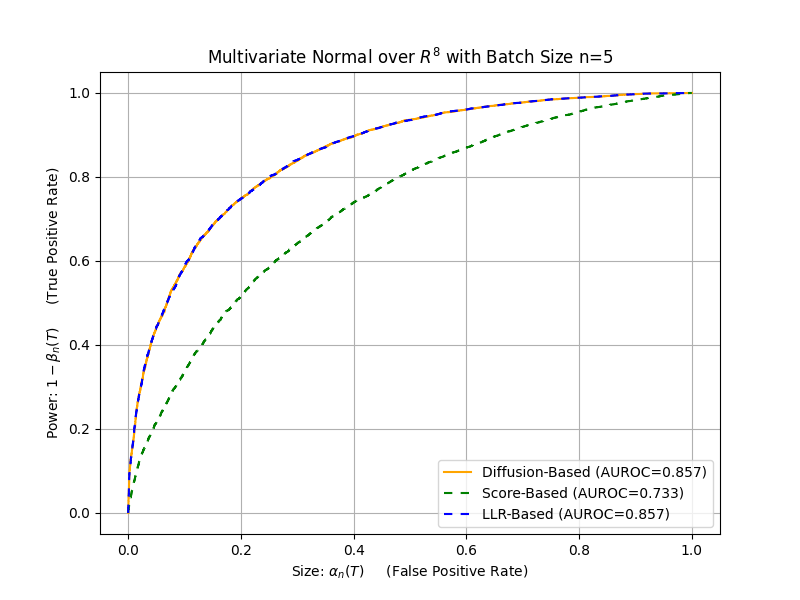} &
\rocimg{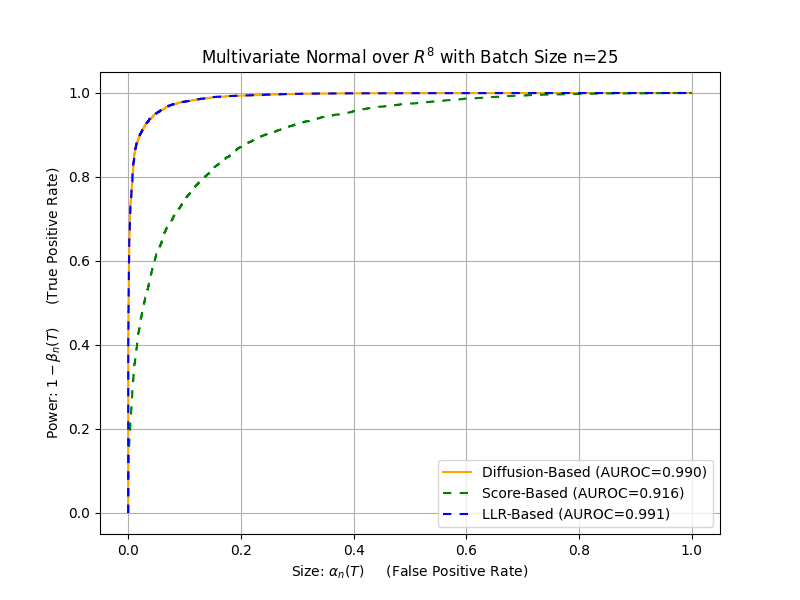} &
\rocimg{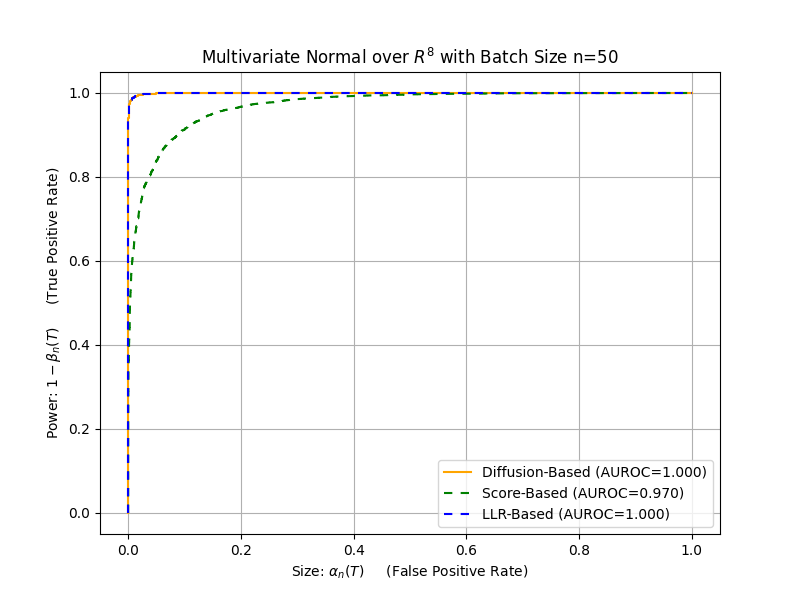} \\[-0.6em]

\rowlab{GB-RBM} &
\rocimg{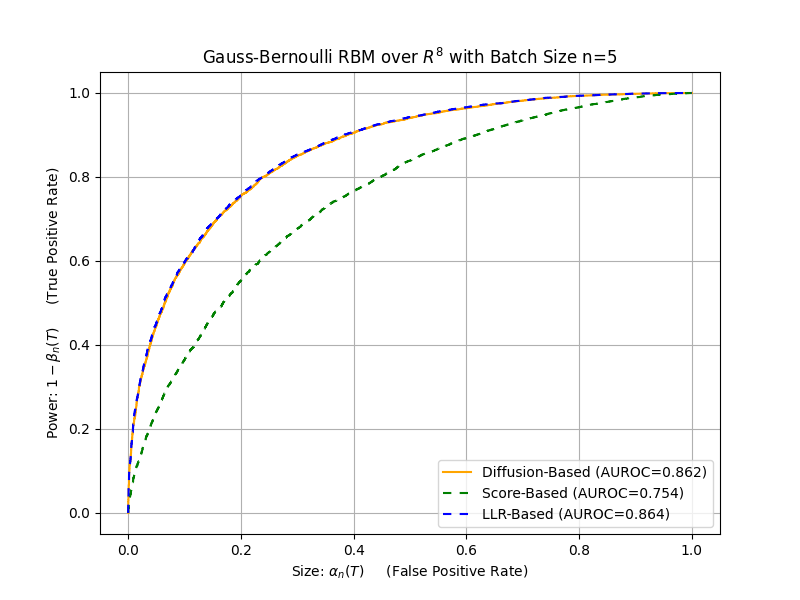} &
\rocimg{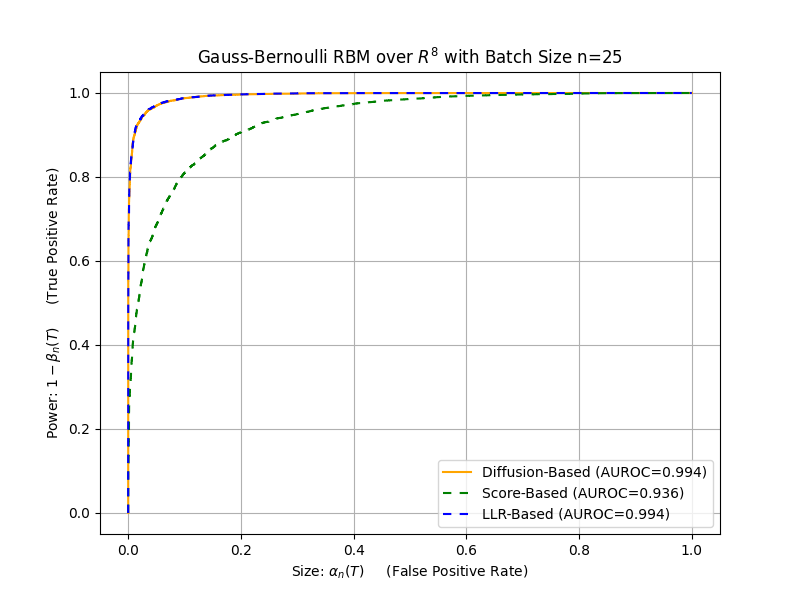} &
\rocimg{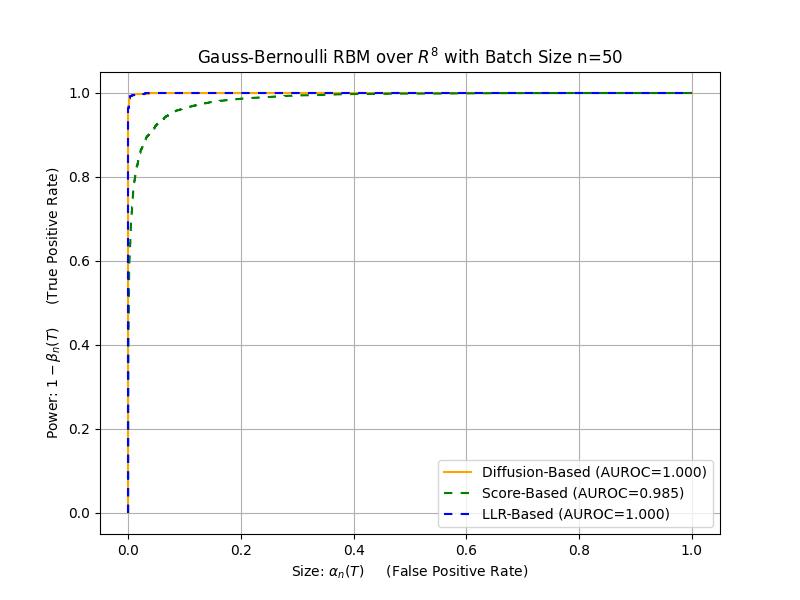} \\[-0.6em]

\rowlab{Quartic Exponential} &
\rocimg{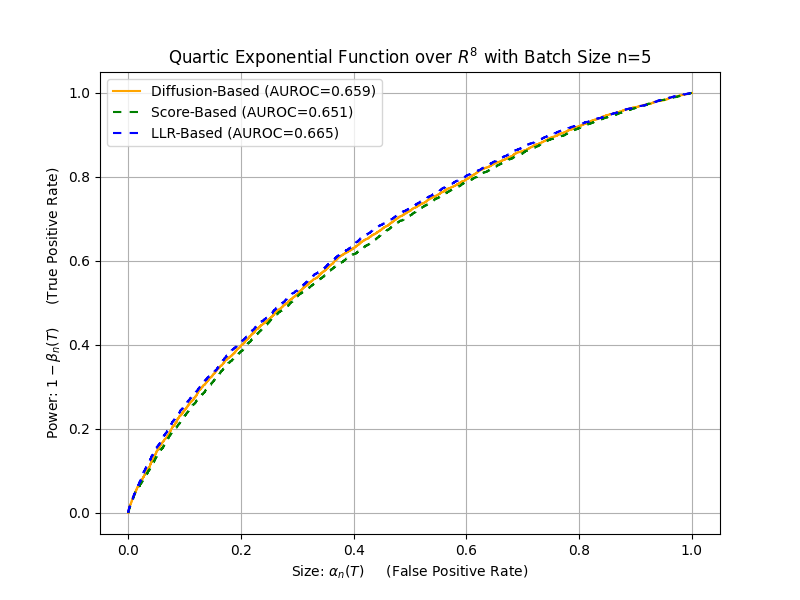} &
\rocimg{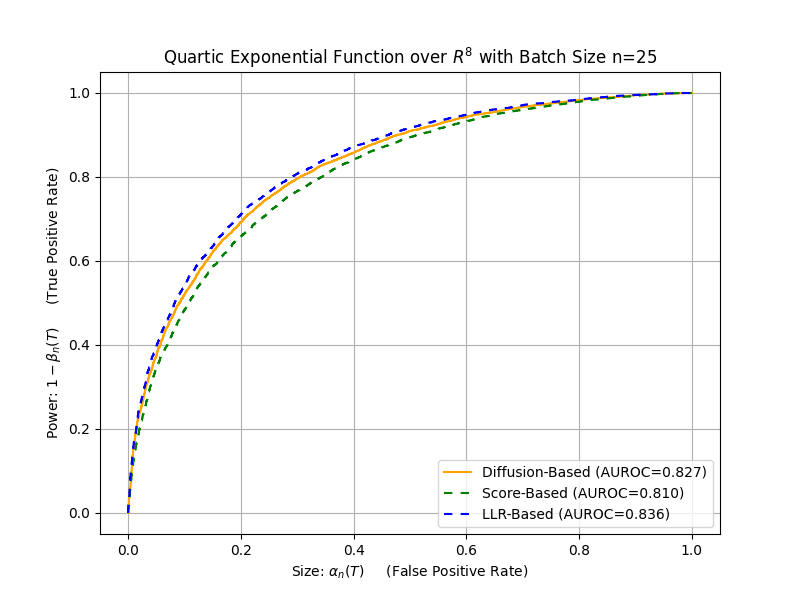} &
\rocimg{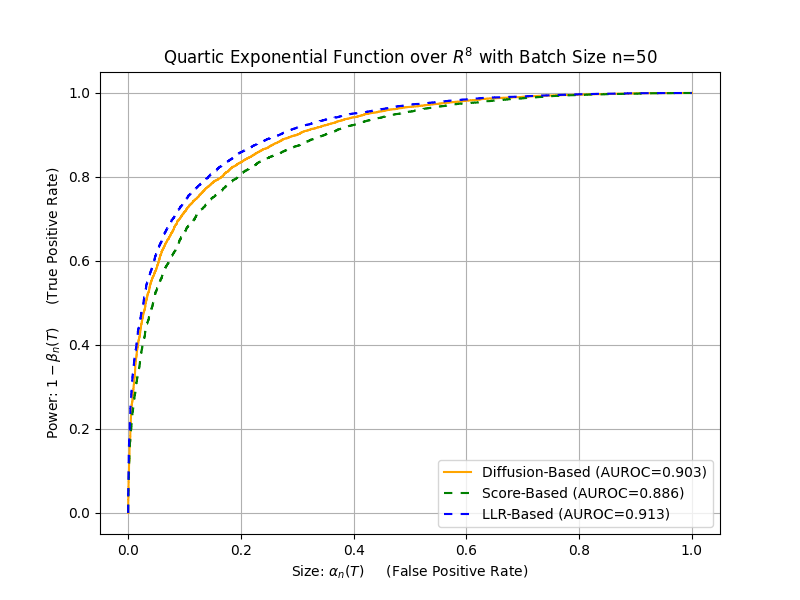} \\[0.35em]

& \small $n=5$ & \small $n=25$ & \small $n=50$

\end{tabular}

\vspace{-0.5em}
\caption{ROC curves for diffusion-based, score-based, and LLR-based hypothesis tests, plotted for batch sizes $n \in \{ 5, 25, 50\}$.  $P_\infty, P_1$ are chosen to follow Gaussian Distributions (top), Gauss-Bernoulli Restricted Boltzmann Machine Distributions (middle), and Quartic Exponential Distributions (bottom).  Each ROC curve uses 10,000 batches.  Note that similar plots for $n \in \{ 1, 10, 100\}$ are provided in Figure~\ref{fig:roc} in the main text.}
\label{fig:roc_extra}
\end{figure*}

\begin{table*}[!t]
\begin{equation}
\label{eq:param_V}
V^* = 
\begin{bmatrix}
6.94357  &   -3.41203  &   -2.15460  &   -0.48852  &   -0.21851  &      -0.39300  &   -0.93257  &   -0.75584 \\
-3.41203  &    3.78724  &    0.5144   &   -0.30651  &    1.64793  &      0.06043  &    0.71543  &   -1.44385 \\
-2.15460  &    0.5144   &    3.75500  &    2.00786  &   -1.22796  &      -0.94496  &   -2.25916  &    0.8728 \\ 
-0.48852  &   -0.30651  &    2.00786  &    2.93120  &   -1.57410  &      -1.91590  &   -1.7714   &    0.02425 \\
-0.21851  &    1.64793  &   -1.22796  &   -1.57410  &    5.37965  &      2.21935  &   -1.66047  &   -2.40907 \\
-0.39300  &    0.06043  &   -0.94496  &   -1.91590  &    2.21935  &      6.24591  &   -0.93225  &    3.02939 \\
-0.93257  &    0.71543  &   -2.25916  &   -1.7714   &   -1.66047  &      -0.93225  &    8.12932  &    0.29485 \\
-0.75584  &   -1.44385  &    0.87281  &    0.02425  &   -2.40907  &      3.02939  &    0.29485  &    6.82808 \\
\end{bmatrix}
\end{equation}

\begin{equation}
\label{eq:param_mu_1}
\mu_1^* = 
\begin{bmatrix}
    0, 0, 0, 0, 0, 0, 0, 0
\end{bmatrix}^T
\end{equation}
\begin{equation}
\label{eq:param_mu_infty}
\mu_\infty^* = 
\begin{bmatrix}
    0.99974, -1.11210, -0.11677,  0.1231 , -0.55111,
    0.29397, -0.71772,  0.93254
\end{bmatrix}^T
\end{equation}
\end{table*}

\section{Proofs of Theorems and Lemmas}
\label{appendix:proofs}

We begin by presenting Lemmas which shall assist in the proofs of the Lemmas and Theorems of this paper.
\begin{lemma}
\label{lemma:hyv_ibp}
For completeness, we recall this Lemma and proof from \cite{hyvarinen2005estimation}[Lemma 4].

For differentiable $g,h : \mathbb{R}^d \mapsto \mathbb{R}$:
\begin{align*}
     \int^\infty_{-\infty} h(X)  \frac{\partial g(X)}{\partial X_1}  dX_1 
     &= \lim_{a \rightarrow \infty,  b\rightarrow -\infty} \bigg(g(a, X_2, \cdots, X_n)h(a, X_2, \cdots, X_n) - g(b, X_2, \cdots, X_n)h(b, X_2, \cdots, X_n)\bigg)  \\
     &\quad\quad\quad\quad\quad - \int^\infty_{-\infty} g(X) \frac{\partial h(X)}{\partial X_1}  dX_1  \eqnum
\end{align*}
The same follows for all $X_i \neq X_1$.
\end{lemma}
\begin{proof}
By the product rule:
\begin{equation}
    \frac{\partial g(X) h(X)}{\partial X_1} = g(X) \frac{\partial h(X)}{\partial X_1} + h(X) \frac{\partial g(X)}{\partial X_1}
\end{equation}
Rearranging:
\begin{equation}
\label{eq:lemma_hyv_final}
    h(X) \frac{\partial g(X)}{\partial X_1} = \frac{\partial g(X) h(X)}{\partial X_1} - g(X) \frac{\partial h(X)}{\partial X_1} 
\end{equation}
The result follows from integrating both sides of \eqref{eq:lemma_hyv_final} over $\mathbb{R}^d$.
\end{proof}

\begin{lemma}
\label{lemma:fubini}
Let $s(X) : \mathbb{R}^d \mapsto \mathbb{R}^d$ be a differentiable function with continuous derivatives and let $P$ be a probability distribution with density $p$ absolutely continuous with respect to the Lebesgue measure.  Let $(v)_i$ denote the $i$-th element of vector $v$. 

If
\begin{equation}
    \mathbb{E}_P[\| \nabla \log p(X) \|^2] < \infty \; \text{ and } \; \mathbb{E}_P[\| s(X) \|^2] < \infty,
\end{equation}
then
\begin{equation}
    \mathbb{E}_P[| (\nabla \log p(X))_i (s(X))_i |] < \infty
\end{equation}
for all $1 \leq i \leq d$.
\end{lemma}
\begin{proof}
By the definition of the norm, we can say that:
\begin{align*}
    &\mathbb{E}_P[ \| \nabla \log p(X) \|^2] = \mathbb{E}_P \bigg[ \sum_i( (\nabla \log p(X))_i)^2 \bigg] \eqnum
\end{align*}
so $\mathbb{E}_P[\| \nabla \log p(X)\|^2 ] < \infty$ implies that $\mathbb{E}_P [ (( \nabla \log p(X))_i )^2] < \infty$ for all $1 \leq i \leq d$.  By similar reasoning, the assumptions of this Lemma imply that $\mathbb{E}_P[((s(X))_i)^2] < \infty$ for all $1 \leq i \leq d$.

By the Cauchy-Schwarz inequality, we have that:
\begin{align*}
    &\mathbb{E}_P[|(\nabla \log p(X))_i (s(X))_i|] \leq \sqrt{\mathbb{E}_P [((\nabla \log p(X))_i)^2] } \sqrt{\mathbb{E}_P[((s(X))_i)^2]} \eqnum
\end{align*}
and thus the quantity of interest is bounded from above by finite terms.

For the special choice of $s(X) = \nabla \log q(X)$, then Assumption~\ref{assumption:hyvarinen_regularity} implies the assumptions of this Lemma, and this Lemma demonstrates that
\begin{equation}
    \mathbb{E}_P[|(\nabla \log p(X))_i (\nabla \log q(X))_i|] < \infty
\end{equation}
for all $1 \leq i \leq d$.
\end{proof}

\begin{lemma}
\label{lemma:hyvarinen}
For $m(X), P_\infty, P_1$ satisfying Assumptions~\ref{assumption:hyvarinen_regularity} and \ref{assumption:p_inf_p_1}, we have that:
\begin{align} \mathbb{E}_P\bigg[\frac{1}{2}\| m^T(X) \big( \nabla \log p(X) - \nabla \log q(X) \big)\|^2 \bigg]
= 
\mathbb{E}_P\bigg[ \frac{1}{2} \|m^T(X) \nabla \log p(X) \|^2 + \mathcal{S}_m(X, Q) \bigg]
\end{align}
\end{lemma}
\begin{proof}[Proof of Lemma~\ref{lemma:hyvarinen}]
This proof follows the arguments of  \cite{hyvarinen2005estimation}[Theorem 1], with modifications for the present setting.  We include the details for completeness.  
For generality, we prove this theorem for arbitrary $s(X)$, but note that when we let $s(X) = \nabla \log q(X)$, then this theorem calculates a new form for  $\mathbb{D}_m(P \| Q)$.

Let $s(X) : \mathbb{R}^d \mapsto \mathbb{R}^d$ be differentiable with a continuous derivative.    We calculate:
\begin{align*}
\label{eq:hyvarinen_expanded}
&\mathbb{E}_P\bigg[\frac{1}{2}\| m^T(X) \big( \nabla \log p(X) - s(X) \big)\|^2\bigg] \\
&= \mathbb{E}_P\bigg[\frac{1}{2}\| m^T(X) \nabla \log p(X) \|^2 + \frac{1}{2} \| m^T(X) s(X) \|^2 
- \underbrace{(m^T(X) \nabla \log p(X))^T m^T(X) s(X)}_{\text{Term 1}}
    \bigg]. \eqnum
\end{align*}

We examine Term 1 in more detail:
\begin{align*} 
\label{eq:hyv_full_expr}
\mathbb{E}_P[(m^T(X) \nabla \log p(X))^T m^T(X) s(X)] =\mathbb{E}_P[(\nabla \log p(X))^T m(X) m^T(X) s(X)]. \eqnum
\end{align*}
We define
\begin{equation}
    f(X) = m(X)m^T(X) s(X).
\end{equation}
Letting $v_i$ and $(v)_i$ denote the $i$-th element of a vector $v$, we can simplify the expression of \eqref{eq:hyv_full_expr}:
\begin{align*}
\mathbb{E}_P[(\nabla \log p(X))^T f(X)] &=\int_{\mathbb{R}^d} p(X) (\nabla \log p(X))^T f(X) dX \\
&= \int_{\mathbb{R}^d} p(X) \sum_{i=1}^d (\nabla \log p(X))_i (f(X))_i dX \\
\label{eq:hyv_integrand}
&= \sum_{i=1}^d \int_{\mathbb{R}^d} p(X)  (\nabla \log p(X))_i (f(X))_i dX \eqnum
\end{align*}

We next choose to calculate the integrand of \eqref{eq:hyv_integrand} for the case of $i=1$:
\begin{align*}
\int_{\mathbb{R}^d} p(X) (\nabla \log p(X))_1 (f(X))_1 dX 
&= \int_{\mathbb{R}^d} p(X) \frac{\partial \log p(X)}{\partial X_1} (f(X))_1 dX \\
&= \int_{\mathbb{R}^d} p(X) \frac{\frac{\partial p(X)}{\partial X_1}}{p(X)} (f(X))_1 dX\\
&= \int_{\mathbb{R}^d} \frac{\partial p(X)}{\partial X_1} (f(X))_1 dX \eqnum
\label{eq:hyv_all_X}
\end{align*}

By the result of Lemma~\ref{lemma:fubini}, we can invoke Fubini's Theorem to expand the integral of \eqref{eq:hyv_all_X} into a double integral:
\begin{align}
\int_{\mathbb{R}^d} \frac{\partial p(X)}{\partial X_1} (f(X))_1 dX 
&= \int_{\mathbb{R}^{d-1}} \underbrace{\bigg( \int^{\infty}_{-\infty} 
\frac{\partial p(X)}{\partial X_1} (f(X))_1 dX_1\bigg)}_{\text{Term 2}} d(X_2, \cdots X_n).
\end{align}
Next, we apply Lemma~\ref{lemma:hyv_ibp} to Term 2, letting $p(X)$ play the role of $g(X)$ and $(f(X))_1$ play the role of $h(X)$.  The limit of Lemma~\ref{lemma:hyv_ibp} evaluates to zero by Part~\ref{item:goes_to_zero_large_x} of Assumption~\ref{assumption:hyvarinen_regularity}.
\begin{align*}
& \int_{\mathbb{R}^{d-1}} \bigg(0 - \int^{\infty}_{-\infty} p(X) \frac{\partial (f(X))_1}{\partial X_1} dX_1 \bigg)d(X_2, \cdots, X_n) = -\int_{\mathbb{R}^d}    p(X) \frac{\partial (f(X))_1}{\partial X_1} dX \eqnum
\end{align*}
where in the last step we collapse the double integral into one single integral over $X$.  

Altogether, we have that:
\begin{align*}
 \mathbb{E}_P[(\nabla \log p(X))^T f(X)] 
&=\sum_{i=1}^d \int_{\mathbb{R}^d} p(X) (\nabla \log p(X))_i (f(X))_i dX \\
&= - \sum_{i=1}^d \int_{\mathbb{R}^d} p(X) \frac{\partial (f(X))_i}{\partial X_i} dX \\
&= -\int_{\mathbb{R}^d} p(X) \sum_{i=1}^d \frac{\partial (f(X))_i}{\partial X_i} dX \\
&= -\mathbb{E}_P[\nabla \cdot f(X)] \eqnum
\end{align*}

Returning to \eqref{eq:hyvarinen_expanded}, we substitute and arrive at:
\begin{align*}
&\mathbb{E}_P\bigg[\frac{1}{2} \| m^T(X) (\nabla \log p(X) -s(X))\|^2 \bigg] \\
&= \mathbb{E}_P \bigg[ \frac{1}{2} \|m^T(X) \nabla \log p(X) \|^2 + \frac{1}{2} \| m^T(X) s(X)\|^2 + \nabla \cdot m(X)m^T(X) s(X)\bigg]. \eqnum
\end{align*}

For the special case where $s(X) = \nabla \log q(X)$, we have that:
\begin{align*}
\mathbb{D}_m(P \|Q) &= \mathbb{E}_P\bigg[\frac{1}{2} \| m^T(X) (\nabla \log p(X) -\nabla \log q(X)) \|^2 \bigg] \\
&= \mathbb{E}_P \bigg[ \frac{1}{2} \|m^T(X) \nabla \log p(X) \|^2 + \frac{1}{2} \| m^T(X) \nabla \log q(X)\|^2 + \nabla \cdot m(X)m^T(X) \nabla \log q(X)\bigg] \\
&= \mathbb{E}_P\bigg[ \frac{1}{2} \| m^T(X) \nabla \log p(X) \|^2 + \mathcal{S}_m(x, Q)\bigg]  \eqnum
\end{align*}

\end{proof}

\begin{proof}[Proof of Lemma~\ref{lemma:hyvarinen_fisher}]
The assumptions of this Lemma are sufficient to guarantee Assumption~\ref{assumption:hyvarinen_regularity} when $m(X)=I$, and the proof follows from the result of Lemma~\ref{lemma:hyvarinen} when $m(X) = I$.
\end{proof}

\begin{proof}[Proof of Lemma \ref{lemma:drifts}]
This proof closely follows the arguments presented in \cite[Lemma 1]{wu_scusum_aistats}, with modification for the present setting.  We include the details for completeness. 

\nocite{lai1998information}

We know that 
\begin{equation}
    \mathbb{D}_m(P \| Q) = \mathbb{E}_{ P} \bigg[ \frac{1}{2} \| m^T(X) \nabla \log p(X) \|^2 + \mathcal{S}_m(X, Q) \bigg]
\end{equation}
We observe that the first term inside the expectation does not depend upon $Q$.  We denote 
\begin{equation}
C_m(R) \de \mathbb{E}_R\bigg[\frac{1}{2} \| m^T(X) \nabla \log r(X)\|^2\bigg]
\end{equation}
for any distribution $R$ with density $r$.  Then:
\begin{align*}
\mathbb{E}_\infty[\mathcal{S}_m(X, P_\infty) - \mathcal{S}_m(X, P_1)] = \mathbb{D}_m(P_\infty \| P_\infty) - C_m(P_\infty) - \mathbb{D}_m(P_\infty \| P_1) + C_m(P_\infty)  = -\mathbb{D}_m(P_\infty \| P_1)
\end{align*}
and 
\begin{align*}
&\mathbb{E}_1 [\mathcal{S}_m(X, P_\infty) - \mathcal{S}_m(X, P_1)]   = \mathbb{D}_m(P_1 \| P_\infty) - C_m(P_1) - \mathbb{D}_m(P_1 \| P_1) + C_m(P_1)   = \mathbb{D}_m(P_1 \| P_\infty).
\end{align*}
\end{proof}

\begin{proof}[Proof of Lemma \ref{lemma:finite_diffusion_divergence}]
Letting 
$$\overline{r}(X) \de m^T(X) \nabla \log r(X),\quad \overline{s}(X) \de m^T(X) \nabla \log s(X),$$
then $\mathbb{D}_m(R \| S) = \frac{1}{2}\mathbb{E}_R[\|\overline{r}(X) - \overline{s}(X)\|^2]$.  
Then, by the triangle inequality,
$$\|\overline{r}(X)-\overline{s}(X)\| \leq \|\overline{r}(X)\| + \|\overline{s}(X)\|$$
and as the norm is nonnegative,
\begin{equation} % Return To Multline in Two Column Mode 
\|\overline{r}(X)-\overline{s}(X)\|^2 \leq (\|\overline{r}(X)\| + \|\overline{s}(X)\|)^2  
\maybeNewline
= \|\overline{r}(X)\|^2 + \|\overline{s}(X)\|^2 + 2\|\overline{r}(X)\|\| \overline{s}(X)\|.
\end{equation} % Return To Multline in Two Column Mode
Finally,
$$\mathbb{E}_R[\|\overline{r}(X) \| \|\overline{s}(X)\|] \leq \sqrt{\mathbb{E}_R[\|\overline{r}(X)\|^2] \mathbb{E}_R[\| \overline{s}(X)\|^2]}$$
by the Cauchy-Schwarz Inequality.  The Lemma follows from Part~\ref{item:finite} of Assumption~\ref{assumption:hyvarinen_regularity}.
\end{proof}

\begin{proof}[Proof of Theorem \ref{theorem:stein_diffusion}]

This proof follows the arguments of the Stein Diffusion Lemma as presented in \cite[Theorem 11.1]{mit_hypo_test} (forward direction only).

Fix some arbitrarily small $\delta > 0$ and set $c = n(\delta-\mathbb{D}_m(P_\infty \| P_1))$.  We recall Definition \ref{definition:diffusion_test}:
\begin{equation}
\label{eq:proof_test}
T_m^c(( X_i)_{i=1}^n) = \begin{cases} 
      0 & \text{ if } \sum_{i=1}^n Z_m(X_i) < c \\
      1 & \text{ else. }
\end{cases}
\end{equation}
Recalling \eqref{eq:err_prob}, the type I error probability of $T_m^c$ is given by
\begin{equation}
    \alpha_n(T_m^c) = \mathbb{P}_\infty \bigg[ \sum_{i=1}^n Z_m(X_i) \geq c \bigg].
\end{equation}
By the law of large numbers, we know that
\begin{equation}
    \frac{1}{n} \sum_{i=1}^n Z_m(X_i)\bigg\vert_{X_i \sim P_\infty} \longrightarrow \mathbb{E}_\infty [Z_m(X)]
\end{equation}
as $n \rightarrow \infty$.  we have demonstrated in \eqref{eq:drifts} that
\begin{equation}
    \mathbb{E}_\infty [Z_m(X)] = -\mathbb{D}_m(P_\infty \| P_1).
\end{equation}
For the $\delta > 0$ previously chosen, the definition of convergence guarantees that for all $\epsilon > 0$, there exists some $N$ for which $n > N$ implies that
\begin{equation}
\label{eq:convergence_t1e}
\mathbb{P}_\infty \bigg[ \bigg\vert \frac{1}{n}\sum_{i=1}^n Z_m(X_i) - (-\mathbb{D}_m(P_\infty \| P_1))\bigg\vert  \geq \delta \bigg] < \epsilon.
\end{equation}
The inequality of \eqref{eq:convergence_t1e} implies that
\begin{equation}
\mathbb{P}_\infty \bigg[ \bigg( \frac{1}{n}\sum_{i=1}^n Z_m(X_i) - (-\mathbb{D}_m(P_\infty \| P_1))\bigg)  \geq \delta \bigg] < \epsilon.
\end{equation}
and hence that
\begin{equation}
    \mathbb{P}_\infty \bigg[ \frac{1}{n}\sum_{i=1}^n Z_m(X_i) \geq (\delta - \mathbb{D}_m(P_\infty \|P_1)) \bigg] < \epsilon.
\end{equation}
Rearranging the inequality and substituting $c = n(\delta - \mathbb{D}_m(P_\infty \| P_1))$, we can say that for all $\epsilon > 0$, there exists an $N$ such that $n > N$ implies:
\begin{equation}
        \mathbb{P}_\infty \bigg[ \sum_{i=1}^n Z_m(X_i) \geq c \bigg] < \epsilon.
\end{equation}
Letting $\epsilon = \overline{\alpha}$ establishes that $\lim_{n \rightarrow \infty } \alpha_n(T_m^c) \leq \overline{\alpha}$ for all $\overline{\alpha} \in (0, 1)$.

Next, we calculate the type II error probability. Recalling \eqref{eq:err_prob}, the type II error probability of $T_m^c$ is given by:
\begin{equation}
    \beta_n(T_m^c) \de \mathbb{P}_1 \bigg[ \sum_{i=1}^n Z_m(X_i) < c \bigg].
\end{equation}
We manipulate and apply the Chernoff Bound:
\begin{align*}
    \mathbb{P}_1 \bigg[ \bigg( -\sum_{i=1}^n Z_m(X_i) \bigg) >  (-c) \bigg] 
    &\leq  \inf_{\theta > 0} e^{-\theta (-c)} \prod_{i=1}^n \mathbb{E}_1[ \exp( -\theta Z_m(X_i))] \\
    &\leq e^{c} \prod_{i=1}^n \mathbb{E}_1 [\exp (- Z_m(X_i))]. \eqnum
\end{align*}
We assume that $\mathbb{E}_1 [\exp( -Z_m(X_i))] \leq 1$.  Thus:
\begin{equation}
    e^{c} \prod_{i=1}^n \mathbb{E}_1 [\exp (-Z_m(X_i))] \leq e^{c},
\end{equation}
and
\begin{equation}
    \beta_n(T_m^c) \leq e^{c} =  \exp\bigg( n(\delta-\mathbb{D}_m(P_\infty \| P_1))\bigg).
\end{equation}

We have that $\mathcal{B}(T_m^c) \geq \mathbb{D}_m(P_\infty \| P_1) - \delta$, where $\mathcal{B}(\cdot)$ is defined in Definition~\ref{def:error_exponent}.
As this result holds for all $\delta > 0$, it follows that $ \mathcal{B}(T_m^c) \geq \mathbb{D}_m(P_\infty \| P_1)$.
\end{proof}

\begin{proof}[Proof of Theorem \ref{theorem:arl}]
This proof closely follows the arguments of Theorem 3 in \cite[Theorem 3]{wu_scusum_aistats}, with modifications for the present setting.  We include the details for completeness. 
We prove in many steps:

\textit{Construction of a Random Walk and Martingale:}
We first define 
\begin{equation}
    \delta \de -\log(\mathbb{E}_\infty[\exp Z_m(X)]).
\end{equation}
From the condition of Theorem \ref{theorem:arl}, $\delta \geq 0$.  Next, define 
\begin{equation}
    \tilde{Z}_m(X) \de Z_m(X) + \delta
\end{equation}
where $Z_m(X)$ follows Definition \ref{def:diffusion_inst_det_score}.
We further define
\begin{align*}
    \tilde{W}_a^b &\de \sum_{i=a}^b \tilde{Z}_m(X_i) \\
    \tilde{G}(n) &\de \exp \tilde{W}_1^n = \exp \bigg( \sum_{i=1}^n \tilde{Z}_m(X_i) \bigg)
\end{align*}
We observe that $\tilde{W}$ is a random walk that can take negative values.
Using Lemma \ref{lemma:drifts} and Jensen's Inequality with respect to the convex function $-\log(\cdot)$:
\begin{equation}
\label{eq:jensen}
      \mathbb{D}_m(P_\infty \| P_1) = -\mathbb{E}_\infty [Z_m(X)] > -\log E_\infty[\exp Z_m(X)] = \delta.
\end{equation}
We present Jensen's Inequality as a strict inequality (it could achieve equality only if $Z_m(X)$ is almost surely constant, but this would imply that $D_m(P_\infty \| P_1) = \mathbb{D}_m(P_1 \| P_\infty) = 0$, violating Assumption~\ref{assumption:positive_diffusion_divergence}).  
Using \eqref{eq:jensen}, we observe:
\begin{align*}
    \mathbb{E}_\infty [\tilde{Z}_m(X)] &= \mathbb{E}_\infty[Z_m(X) + \delta] 
 < -\mathbb{D}_m(P_\infty \| P_1) + \mathbb{D}_m(P_\infty \| P_1) = 0. \eqnum
\end{align*}
As $\mathbb{E}_\infty [\tilde{Z}_m(X)] < 0$, we note that $\tilde{W}$ is a random walk with a strictly negative drift.

We observe that
\begin{equation} % Return To Multline in Two Column Mode
    \mathbb{E}_\infty[\tilde{G}(n+1) | \mathcal{F}_n] = \tilde{G}(n) \mathbb{E}_\infty [\exp (\tilde{Z}_m(X_{n+1}))] \\
    = \tilde{G}(n) e^\delta \mathbb{E}_\infty[\exp Z_m(X_{n+1})] = \tilde{G}(n),
\end{equation} % Return To Multline in Two Column Mode
and that
\begin{align*}
    \mathbb{E}_\infty[\tilde{G}(n)] &= \mathbb{E}_\infty \bigg[ \exp\bigg( 
\sum_{i=1}^n (Z_m(X_i) + \delta)  \bigg) \bigg] 
= e^{n\delta} \prod_{i=1}^n \mathbb{E}_\infty [\exp Z_m(X_i)] = 1, \eqnum
\end{align*}
so $\tilde{G}$ is a non-negative martingale with mean one under $P_\infty$.

\textit{Construction of Stopping Rule $\tilde{\tau}$:}
Next, we define a stopping time:
\begin{equation}
    \tilde{\tau} = \inf \bigg\{ n \geq 1 : \bigg(\max_{1 \leq k \leq n} \tilde{W}_{k}^n \bigg) \geq c \bigg\}
\end{equation}
We note that $\tilde{\tau}$ cannot be trivially calculated due to the strictly negative drift of the random walk given by $\tilde{W}_k^n$.

\textit{Construction of Stopping Rule $M$:}
We define a sequence of stopping times:
\begin{align*}
    \eta_0 &\de 0,\quad\quad 
    \eta_1 \de \inf \bigg\{ t : \tilde{W}_1^t < 0 \bigg\},\quad\quad 
    \eta_{k+1}  \de \inf \bigg\{ t > \eta_k : \tilde{W}_{\eta_k + 1}^t < 0 \bigg\},
\end{align*}
and
\begin{equation} % Return To Multline in Two Column Mode
    M \de \inf \bigg\{ k \geq 0 : \eta_k < \infty 
    \maybeNewline
    \text{ and } \tilde{W}_{\eta_k + 1}^n >c \text{ for some } n > \eta_k \bigg\}.
\end{equation} % Return To Multline in Two Column Mode
We can see that $M \leq \tilde{\tau}$.  Since $\tilde{Z}_m(X) \geq Z_m(X)$, we know that $\tilde{\tau} \leq \tau$.  Hence, $\mathbb{E}_\infty [\tau] \geq \mathbb{E}_\infty[M]$.  

\textit{Calculation of $\mathbb{P}_\infty(M > k)$:}
As an intermediate step in the calculation of $\mathbb{E}_\infty[M]$, we first calculate $\mathbb{P}_\infty(M > k)$:
\begin{align*}
    \mathbb{P}_\infty [M > k] = \mathbb{E}_\infty[\mathbbm{1}_{\{M > k\}}] = \mathbb{E}_\infty[\mathbbm{1}_{\{M > k\}} \mathbbm{1}_{\{M \geq k\}}] 
    &= \mathbb{E}_\infty[\mathbb{E}_\infty[\mathbbm{1}_{\{M > k\}} \mathbbm{1}_{\{M \geq k\}} | \mathcal{F}_{\eta_k} ]]\\
    &= \mathbb{E}_\infty[\mathbb{P}(M \geq k + 1 | \mathcal{F}_{\eta_k}) \mathbbm{1}_{\{M \geq k\}} ]. \\
\end{align*}
We next consider the probability $\mathbb{P}_\infty (M \geq k + 1 | \mathcal{F}_{\eta_k})$. 
\begin{equation}
\label{eq:m_k_1}
    \mathbb{P}_\infty(M \geq k + 1 | \mathcal{F}_{\eta_k}) 
     =1-\mathbb{P}_\infty (M \leq k | \mathcal{F}_{\eta_k}).
\end{equation}
We calculate $\mathbb{P}(M \leq k | \mathcal{F}_{\eta_k})$:
\begin{align*}
    \mathbb{P}_\infty \bigg( \tilde{W}_{\eta_k + 1}^n > c \text{ for some } n > \eta_k \bigg| \mathcal{F}_{\eta_k} \bigg)  
    &= \mathbb{P}_\infty \bigg( \tilde{W}_1^n > c \text{ for some } n \bigg) \\
    &= \lim_{t \rightarrow \infty} \mathbb{P}_\infty \bigg( \bigg( \max_{n : n \leq t} \tilde{W}_1^n\bigg) \geq c \bigg) \\
    & = \lim_{t \rightarrow \infty} \mathbb{P}_\infty \bigg( \bigg(\max_{n: n \leq t} \exp(\tilde{W}_1^n) \bigg) > e^c \bigg) \\
    & = \lim_{t \rightarrow \infty} \mathbb{P}_\infty \bigg(\bigg(\max_{n:n\leq t} \tilde{G}(n)\bigg) > e^c \bigg) \\
    & \leq \frac{\mathbb{E}_\infty[\tilde{G}(t)]}{e^c} = e^{-c},
\end{align*}
where in the first step we rely upon the i.i.d. (and hence stationary) nature of $X_i$ and in the last step we invoke Doob's submartingale inequality (\cite{doob1953stochastic}), noting that $\tilde{G}(t)$ is a nonnegative martingale with mean one under $P_\infty$.

From \eqref{eq:m_k_1}, we know that $\mathbb{P}(M \geq k + 1 | \mathcal{F}_{\eta_k}) \geq 1-e^{-c}$ and hence that
\begin{align*}
    \mathbb{P}_\infty[M > k] 
    &= \mathbb{E}_\infty [ \mathbb{P}_\infty [M  \geq k + 1| \mathcal{F}_{\eta_k}] \mathbbm{1}_{ \{ M \geq k \} } ] \\
    &  \geq  (1-e^{-c}) \mathbb{P}_\infty [M > k -1] \\
    &  \geq (1-e^{-c})^2 \mathbb{P}_\infty [M > k -2] \\
    &  \geq \cdots \geq (1-e^{-c})^k .
\end{align*}
We next sum a geometric series:
\begin{equation} % Return To Multline in Two Column Mode
    \mathbb{E}_\infty [M] = \sum_{k=0}^\infty \mathbb{P}(M > k) 
    \geq \sum_{k=0}^\infty (1-e^{-c})^k 
    \maybeNewline
    = \frac{1}{1-(1-e^{-c})}= e^c
\end{equation} % Return To Multline in Two Column Mode
and we conclude that
\begin{equation}
    \mathbb{E}_\infty[\tau] \geq \mathbb{E}_\infty[\tilde{\tau}] \geq \mathbb{E}_\infty [M] \geq e^c.
\end{equation}
\end{proof}

\begin{proof}[Proof of Theorem \ref{theorem:edd}]
This proof closely follows the arguments presented in \cite[Theorem 4]{wu_scusum_aistats}, with modifications for the present setting.  We include the details for completeness.

\nocite{woodroofe1982nonlinear}

Define a random walk
\begin{equation}
    \hat{Y}(n) = \sum_{i=1}^n Z_m(X_i), \; n \geq 1
\end{equation}
and a stopping time
\begin{equation}
    \hat{R} \de \inf\{ n \geq 1 : \hat{Y}(n) \geq c \}.
\end{equation}
We define the overshoot of the random walk over threshold $c$ by:
\begin{equation}
    Q_c \de \hat{Y}(\hat{R}) - c.
\end{equation}
Define $\mu_m \de \mathbb{E}_1 [Z_m(X)]$ and $\sigma_m^2 = \text{Var}_1[Z_m(X)]$.  We define
\begin{equation}
    \mu_m \de \mathbb{E}_1 [Z_m(X)] = \mathbb{D}_m(P_1 \| P_\infty),
\end{equation}
where we use the result of Lemma \ref{lemma:drifts}.  We further define
\begin{equation}
    \sigma_m^2 \de \text{Var}_1[Z_m(X)]= \mathbb{E}_1[Z_m(X)^2] - \mu_m^2.
\end{equation}
Under the mild assumption that
\begin{equation}
    \mathbb{E}_1[\mathcal{S}_m(X, P_\infty)^2],\; \mathbb{E}_1[\mathcal{S}_m(X, P_1)^2] < \infty,
\end{equation}
we can use Theorem 1 of \cite{lorden1970excess} to conclude that:
\begin{equation}
    \sup_{c \geq 0} \mathbb{E}_1[Q_c] \leq \frac{\mathbb{E}_1[(Z_m(X)^+)^2]}{\mathbb{E}_1[Z_m(X)]} \leq \frac{\mathbb{E}_1[Z_m(X)^2]} {\mathbb{E}_1[Z_m(X)]} = \frac{\mu_m^2 + \sigma_m^2}{\mu_m},
\end{equation}
where $(\cdot)^+ = \max(0, \cdot)$.  We next use Wald's Lemma \cite{woodroofe1982nonlinear} to show that for all $c > 0$:
\begin{equation}
    \mathbb{E}_1[\hat{R}] = \frac{c}{\mu_m} + \frac{\mathbb{E}_1[Q_c]}{\mu_m} \leq \frac{c}{\mu_m} + \frac{\mu_m^2 + \sigma_m^2}{\mu_m^2}.
\end{equation}
For any $n \geq 0$, $\hat{Y}(n) \leq Y_m(n)$ and so $\hat{R} \geq \tau_m^c$ for the $Y_m(\cdot), \tau_m^c$ of Definition~\ref{def:diffusion_stopping_rule}.  Then:
\begin{equation}
    \mathbb{E}_1[\tau_m^c] \leq \mathbb{E}_1[\hat{R}] \leq \frac{c}{\mu_m} + \frac{\mu_m^2 + \sigma_m^2}{\mu_m^2},
\end{equation}
but as $c \rightarrow \infty$, the effect of the term $(\mu_m^2 + \sigma_m^2) / \mu_m^2$ approaches zero.  
Thus,
\begin{equation}
\mathcal{L}_{\texttt{\textup{WADD}}}(\tau_m^c) \sim \frac{c}{\mathbb{D}_m(P_1 \| P_\infty)}.
\end{equation}
where $f(c) \sim g(c)$ as $c \rightarrow \infty$ means that $\lim_{c \rightarrow \infty } {f(c)}/{g(c)} = 1$.

\end{proof}

\begin{proof}[Proof of Theorem \ref{theorem:gaussian_analytical}]
Direct calculation gives:
\begin{align*}
    Z_{\texttt{\textup{KL}}}(X) =& \log p_1(X) - \log p_\infty(X) = -\frac{1}{2}(X-\mu_1)^T V^{-1} (X-\mu_1)  + \frac{1}{2}(X-\mu_\infty)^T V^{-1} (X-\mu_\infty).
\end{align*}

Without loss of generality, we can substitute $0$ in place of  $\mu_\infty$ and substitute $\mu = \mu_1 - \mu_\infty$ in place of $\mu_1$.  With this substitution:
\begin{equation}
Z_{\texttt{\textup{KL}}}(X) =  X^T V^{-1}\mu -\frac{1}{2} \mu^T V^{-1}\mu  .   %\bigg(x-\frac{\mu}{2}\bigg)^T V^{-1}\mu
\end{equation}
In the same way, we calculate $Z_{M}(\cdot)$ for any fixed matrix $M$:
\begin{equation}
Z_{M}(X) = X^T V^{-1} M M^T V^{-1} \mu 
-\frac{1}{2} \mu^T V^{-1} M M^T V^{-1} \mu.
\end{equation}

In the special case where $M=I$ (Fisher divergence-based setting), the formula for $Z_M(X)$ is the formula for $Z_{\texttt{\textup{KL}}}(X)$ where all instances of $V^{-1}$ are replaced by instances of $V^{-2}$.  In our test, setting $M = V^\frac{1}{2}$ lets $M M^T = V^{1}$, which makes the diffusion-based test equivalent to its LLR-based analog.
\end{proof}

\begin{proof}[Proof of Theorem \ref{theorem:ode}]
Consider $P_\infty =\mathcal{N}(0, 1/4)$ and $P_1 = \mathcal{N}(0,1)$.  Here, $X \in \mathbb{R}^d$ for $d=1$.  There does not exist any matrix-valued function $m(X) : \mathbb{R}^d \mapsto \mathbb{R}^{d \times w}$ such that $Z_m(X) = Z_{\texttt{\textup{KL}}}(X)$ across all $X \in \mathbb{R}$.  Define 
\begin{align}
    P_\infty = \mathcal{N}(0, \sigma^2),\; P_1 = \mathcal{N}(0,1).
\end{align}
These distributions are supported on $\mathbb{R}$.  Suppose for contradiction that there exists some $m(X) : \mathbb{R} \mapsto \mathbb{R}$ such that:
\begin{equation}
    Z_{\texttt{\textup{KL}}}(X) = Z_m(X) \;\; \forall X \in \mathbb{R}.
\end{equation}
where $Z_{\texttt{\textup{KL}}}(X)$ is given in \eqref{eq:test_llr}.  Defining 
\begin{equation}
    u(X) \de m^2(X),
\end{equation}
we can see that $u(X) \geq 0$ for all $X \in \mathbb{R}$.  We shall demonstrate a contradiction by showing that $u(X)$ cannot be positive for all $X \in \mathbb{R}$ when $\sigma = 1/2$.

To simplify notation, we denote
\begin{align}
s_\infty(X) &\de \nabla \log p_\infty(X), \\
s_1(X) &\de \nabla \log p_1(X).
\end{align}

When $X, m(X)$ are scalars, then $s_\infty(X), s_1(X)$ are also scalars, and we can re-write the diffusion-based instantaneous detection score of Definition \ref{def:diffusion_inst_det_score} in simpler terms.  We denote the derivatives of $u(X), s_\infty(X)$, and $s_1(X)$ as $u'(X), s_\infty'(X)$, and $s_1'(X)$, respectively.
\begin{align*}
    Z_{\sqrt{u}}(X) =& \frac{1}{2} u(X) s_\infty^2(X) + \frac{\partial}{\partial X} \bigg(u(X) s_\infty (X)\bigg) - \frac{1}{2}u(X) s_1^2(X) - \frac{\partial}{\partial X} \bigg(u(X) s_1(X)\bigg) \\
    =& \frac{1}{2}u(X) \big(s_\infty^2(X) - s_1^2(X)\big) + u'(X) \big(s_\infty(X) - s_1(X)\big) + u(X) \big(s_\infty '(X) - s_1'(X)\big) \\
    =& u(X) \bigg( \frac{1}{2}(s_\infty^2(X) 
- s_1^2(X)) + (s_\infty'(X) - s_1'(X)\bigg)  + u'(X) \bigg(s_\infty(X) - s_1(X)\bigg). \eqnum
\end{align*}

We next calculate $Z_{\texttt{\textup{KL}}}(X)$.  We note that for a scalar Gaussian distribution $P=\mathcal{N}(\mu, \sigma)$: 
\begin{equation}
    \log p(X) = -\frac{1}{2}\log (2 \pi \sigma^2)-\frac{1}{2 \sigma^2}(X-\mu)^2,
\end{equation}
and so 
\begin{align*}
    &\log p_1(X) - \log p_\infty(X) = -\frac{1}{2} \log (2\pi) - \frac{1}{2}X^2 + \frac{1}{2}\log (2\pi \sigma^2) + \frac{1}{2 \sigma^2}X^2 = \log(\sigma) + \frac{X^2}{2}\bigg(\frac{1}{\sigma^2} - 1\bigg) .\eqnum
\end{align*}

Setting $Z_{\texttt{\textup{KL}}}(X) = Z_m(X)$, we arrive at the first-order ordinary differential equation (ODE):
\begin{equation} % Return To Multline in Two Column Mode
\label{eq:ode}
u(X) \bigg( \frac{1}{2}(s_\infty^2(X) 
- s_1^2(X)) + (s_\infty'(X) - s_1'(X))\bigg) 
\maybeNewline
+ u'(X) \bigg(s_\infty(X) - s_1(X)\bigg) 
\maybeNewline
= \log(\sigma) + \frac{X^2}{2}\bigg(\frac{1}{\sigma^2} - 1\bigg) .
\end{equation} % Return To Multline in Two Column Mode
We calculate:
\begin{align}
\label{eq:score_calculations}
    s_\infty(X) = -\frac{1}{\sigma^2}X,\; s_1(X) = -X,
\end{align}
and further observe that $s_\infty(X) - s_1(X) = (1-\sigma^{-2})X$, that $s_\infty'(X) - s_1'(X) = (1-\sigma^{-2})$, and that $s_\infty^2(X) - s_1^2(X) = (\sigma^{-4} - 1)X^2$.
Plugging into \eqref{eq:ode}:
\begin{align*}
\label{eq:substituted_ode}
&  u(X) \bigg( \frac{1}{2}X^2\bigg(\frac{1}{\sigma^4}-1\bigg) + \bigg(1-\frac{1}{\sigma^2}\bigg)\bigg) + u'(X) \bigg(X\bigg(1-\frac{1}{\sigma^2}\bigg)\bigg) = \log(\sigma) + \frac{X^2}{2}\bigg(\frac{1}{\sigma^2} - 1\bigg). \eqnum
\end{align*}
We next divide through by $(1-\sigma^{-2})X$, noting that this term has a zero at $X=0$.  We shall take care not to integrate the resulting ODE through $X=0$.  Furthermore, this also introduces a zero if $\sigma =1$, though in this case $p_\infty(X) = p_1(X)$ for all $X \in \mathbb{R}$, rendering detection impossible.
Performing this division, we get that
\begin{equation}
    u(X) k(X) + u'(X) = r(X),
\end{equation}
where
\begin{equation}
    k(X) \de \frac{ \frac{1}{2}X^2\bigg(\frac{1}{\sigma^4}-1\bigg) + \bigg(1-\frac{1}{\sigma^2}\bigg)}{X\bigg(1-\frac{1}{\sigma^2}\bigg)} = -\gamma X + \frac{1}{X} \quad \text{ where } \quad \gamma \de\frac{1}{2}\bigg(\frac{1-\frac{1}{\sigma^4}}{1-\frac{1}{\sigma^2}}\bigg),
\end{equation}
and where 
\begin{equation}
    r(X) \de \frac{\log(\sigma) + \frac{X^2}{2}\bigg(\frac{1}{\sigma^2} - 1\bigg) }{X\bigg(1-\frac{1}{\sigma^2}\bigg)} = \frac{\delta}{X}- \frac{X}{2} \quad \text{ where } \quad \delta \de \bigg(\frac{\log \sigma}{1-\frac{1}{\sigma^2}}\bigg).
\end{equation}

We note that $\gamma, \delta$ are independent of $X$ and that they are positive for all $\sigma > 0$.

We shall attempt to solve this ODE via the integrating factor method.  We calculate that $\int k(X) dX = \int -\gamma X + X^{-1}dX = -\frac{\gamma}{2}X^2 + \log |X|$ (note that it is not necessary to include a constant of integration $C$, as any such constant would cancel in the step of \eqref{eq:ode_w_D}).  We define
\begin{align*}
\label{eq:def_D}
    D(X) &\de \exp\bigg(\int k(X) dX \bigg) = \exp\bigg(-\frac{\gamma}{2}X^2 \bigg) |X| . \eqnum
\end{align*}
We recall the ODE of \eqref{eq:substituted_ode} and multiply through by $D(X)$:
\begin{equation}
\label{eq:ode_w_D}
    D(X) u(X) k(X) + D(X) u'(X) = D(X) r(X).
\end{equation}
We observe that $\frac{\partial}{\partial X}D(X) = D(X) k(X)$ and that $\frac{\partial}{\partial X} D(X) u(X) = D(X)k(X)u(X) + D(X)u'(X)$, which is the left-hand-side of \eqref{eq:ode_w_D}.  We make this substitution to the left-hand-side and apply the Fundamental Theorems of Calculus:
\begin{equation}
\label{eq:first_second_ftc}
    D(X)u(X) - D(a)u(a) = \int_a^X D(y)r(y) dy.
\end{equation}
We know by Lemma \ref{lemma:finite_diffusion_divergence} that the diffusion divergence is bounded.  Recalling \eqref{eq:score_calculations}, we can express $\mathbb{D}_m(P_1 \| P_\infty)$ as:
\begin{align*}
     \frac{1}{2}\mathbb{E}_1\bigg[\big\|m^T(X) (\nabla \log p_1(X) - \nabla \log p_\infty(X))\big\|^2\bigg] 
    = \frac{1}{2\sqrt{2\pi }}\bigg(1-\frac{1}{\sigma^2}\bigg)^2 \int_{-\infty}^\infty e^{-\frac{1}{2}X^2} u(X) X^2 dX < \infty .\eqnum
\end{align*}
Consider $\sigma = \frac{1}{2}$.  For this choice of $\sigma$, we have that $\gamma = 2.5 > 1$.  Recalling that $u(X) \geq 0$ for all $X$ and noting that $e^{h(X)}, X^2, |X| \geq 0$ for all $X, h(\cdot)$, we have that:
\begin{align*}
\label{eq:convergence_impl1}
    \int_{-\infty}^\infty e^{-\frac{1}{2}X^2} u(X) X^2 dX < \infty 
    \implies&\int_{-\infty}^\infty e^{-\frac{\gamma}{2}X^2} u(X) X^2 dX < \infty \eqnum
\end{align*}
as $\frac{\gamma}{2} > \frac{1}{2}$.  Furthermore, for all $X \not\in [-1, 1]$, we know that $X^2 > |X|$.  Thus, the convergence of \eqref{eq:convergence_impl1} implies:
\begin{align}
\label{eq:convergence_impl2}
    \implies&\int_{-\infty}^\infty e^{-\frac{\gamma}{2}X^2} u(X) |X| dX < \infty.
\end{align}
We note that the integral of \eqref{eq:convergence_impl2} is equivalent to $ e^{-C}\int_{-\infty}^\infty D(X)u(X) dX$.  As all terms of the integrand are non-negative for all $X$, the convergence is absolute and implies that 
\begin{equation}
\label{eq:limits}
    \lim_{x\rightarrow \infty} D(X) u(X) = \lim_{X \rightarrow -\infty} D(X) u(X) = 0.
\end{equation}
Plugging \eqref{eq:limits} into \eqref{eq:first_second_ftc}, we get that for all $ X<0$:
\begin{equation}
\label{eq:integration_factor_method}
    D(X)u(X) = \int_{-\infty}^X D(y)r(y) dy.
\end{equation}
%We note that in \eqref{eq:integration_factor_method}, both sides of the equation are linear in $e^C$.  Thus, we proceed with our calculation assuming that $e^C$ cancels and omit it in future steps. 
We now expand the integrand:
\begin{align*}
     \int_{-\infty}^{X}D(y)r(y) dy 
     &=\underbrace{ \int_{-\infty}^X  \exp\bigg(-\frac{\gamma}{2} y^2 \bigg)  \frac{|y|}{y} \delta \;dy}_{\text{term 1}} - 
     \underbrace{ \int_{-\infty}^X  \exp\bigg(-\frac{\gamma}{2} y^2\bigg) |y|\frac{y}{2}\; dy}_{\text{term 2}} .
\end{align*}
If $X < 0$, then $y < 0$ and $\frac{|y|}{y} = -1$.  Thus:
\begin{equation}
    \text{term 1} = - \delta \sqrt{\frac{2\pi}{\gamma}} \Phi(X \sqrt{\gamma}) ,
\end{equation}
where $\Phi$ denotes the CDF of a standard scalar Gaussian with mean zero and standard deviation one. We next integrate term 2 using integration by parts, keeping in mind that $X < 0 \implies y < 0$, and hence $y |y| = -y^2$:
\begin{align*}
    \text{term 2} &=\int_{-\infty}^X  e^{-\frac{\gamma}{2}y^2} \frac{(- y^2)}{2} dy\\
    &=\frac{1}{2\gamma}\int_{-\infty}^X  (-\gamma y e^{-\frac{\gamma}{2}y^2})( y) dy\\
    &=\frac{1}{2\gamma}\bigg([\gamma y e^{-\frac{\gamma}{2}y^2}]_{-\infty}^X - \int_{-\infty}^X e^{-\frac{\gamma}{2}y^2} dy \bigg)\\
    &=\frac{1}{2\gamma }\bigg(X e^{-\frac{\gamma}{2}X^2} - \sqrt{\frac{2\pi}{\gamma}}\Phi(X \sqrt{\gamma})\bigg) .\eqnum
\end{align*}

We now know that 
\begin{equation}
    u(X) = \frac{1}{D(X)} \underbrace{\bigg( \zeta \Phi(X \sqrt{\gamma}) 
    - 
    \frac{1}{2\gamma }X e^{-\frac{\gamma}{2}X^2}
    \bigg)}_{\tilde{u}(X)},
\end{equation}
where $\zeta = (-\delta -(2\gamma)^{-1})\sqrt{2\pi \gamma^{-1}}$.    
We observe that $D(X) \geq 0$ for all $X$, and the sign of $u(X)$ is equal to the sign of $\tilde{u}(X)$.  Evaluating, we observe that $\tilde{u}(-1) \approx 0.054 > 0$ but $\tilde{u}(-0.05) \approx -0.013 < 0$.  As such, the $u(X) < 0$ for some $X \in \mathbb{R}$, a contradiction.  No matrix-valued function $m(X) : \mathbb{R} \mapsto \mathbb{R}$ can enforce that $Z_{\texttt{\textup{KL}}}(X)=Z_m(X)$ for all $X \in \mathbb{R}$.

Consider a more general case, where $m(X) : \mathbb{R} \mapsto \mathbb{R}^{1 \times w}$.  Letting $s_P(X) \de \nabla \log p(X)$, the diffusion-Hyv\"arinen score becomes:
\begin{equation} % Return To Multline in Two Column Mode
    \mathcal{S}_m(X, P) = \frac{1}{2}\| m^T(X) s_P(X) \|^2 
    \maybeNewline
    + \nabla \cdot m(X) m^T(X) s_P(X).
\end{equation} % Return To Multline in Two Column Mode
Expanding the first term:
\begin{align*}
    \frac{1}{2}\| m^T(X)  s_P(X) \|^2 
    = &\frac{1}{2} \| (m_1(X)s_P(X) , \cdots m_w(X)s_P(X) )^T \|^2 \\
    = &\frac{1}{2} (m_1^2(X) s_P^2(X) + \cdots + m_w^2(X) s_P^2(X) ) \\
    = &\frac{1}{2} \bigg(\sum_{i=1}^w m_i^2(X) \bigg) s_P^2(X), \eqnum
\end{align*}
as $m_i(X)$ is scalar for any $1 \leq i \leq w$. 
 Expanding the second term:
\begin{align*}
    \nabla \cdot & m(X)m^T(X) s_P(X) = \nabla \cdot \| m^T(X) \|^2 s_P(X) = \nabla \cdot \bigg( \sum_{i=1}^w m_i^2 (X) \bigg) s_P(X) .\eqnum
\end{align*}

In this case, scaling by $m(X): \mathbb{R} \mapsto \mathbb{R}^{1 \times w}$ is equivalent to scaling by $\tilde{m}(X) : \mathbb{R} \mapsto \mathbb{R}$ where
\begin{equation}
    \tilde{m}(X) = \sqrt{\sum_{i=1}^w m_i^2(X)}.
\end{equation}
We know that there cannot exist an $\tilde{m}(X)$ such that the diffusion test statistic matches the log-likelihood ratio. 
Thus, for this choice of $P_\infty, P_1$, for all matrix-valued functions $m(X) : \mathbb{R} \mapsto \mathbb{R}^{1 \times w} $ where $ w \in \mathbb{N}$, $Z_m(\cdot)$ is not equal to $Z_\texttt{\textup{KL}}(\cdot)$.

Define $h(X) = | Z_{\texttt{\textup{KL}}}(X) - Z_m(X) |$.  We know that for all $m(X)$ there exists at least one value 
$X^* \in \mathbb{R}$ for which $h(X^*) \neq 0$. 
By Assumptions~\ref{assumption:p_inf_p_1} \ref{assumption:hyvarinen_regularity}, $Z_m(X)$, $Z_\texttt{\textup{KL}}(X)$, and $h(X)$ are all continuous in $X$.  Letting $\epsilon = h(X^*) / 2$, we have from the definition of continuity that there exists a $\delta > 0$ such that $h(X) > \epsilon$ for all $X \in (X^*-\delta, X^* + \delta)$.  We know from Assumption~\ref{assumption:p_inf_p_1} that $P_\infty, P_1$ are supported on $\mathbb{R}^d$; hence $\mathbb{P}_\infty[X \in (X^*-\delta, X^* + \delta)] > 0$ and $\mathbb{P}_1[X \in (X^*-\delta, X^* + \delta)] > 0$.

\end{proof}

\balance

\end{document}